\documentclass[nofootinbib,twocolumn]{revtex4-1}
\usepackage{graphicx}
\usepackage{amssymb,amsthm}
\usepackage{algorithm}
\usepackage{algpseudocode}

\usepackage{amsfonts}
\usepackage{dsfont}
\usepackage{color}
\usepackage{braket}
\graphicspath{ {Figures/} }

\newcommand{\maxv}{\max_{\mbf{v}}}
\newcommand{\maxh}{\max_{\mbf{h}}}
\newcommand{\maxvh}{\max_{\{\mbf{v},\mbf{h}\}}}

\newcommand{\phvc}{p({\bf h} | {\bf v})}

\newcommand{\rv}{r({\bf v})}

\usepackage[cmex10]{amsmath}
\usepackage[cspex,bbgreekl]{mathbbol}
\usepackage[hang]{subfigure}
\usepackage{yfonts}
\usepackage[inline]{enumitem}
\usepackage{geometry}
\DeclareMathOperator*{\argmax}{arg\,max} 
\DeclareMathOperator*{\argmin}{arg\,min} 
 
\newtheorem{theorem}{Theorem}[section]

\newtheorem{lemma}[theorem]{Lemma}
\newtheorem{proposition}[theorem]{Proposition} 
\newtheorem{corollary}{Corollary}[theorem]
\newtheorem{definition}{Definition}[section]

\newenvironment{remark}{\par\noindent{\bf Remark.\ } \it}

\newcommand{\PP}{\mathrm{Pr}}
\newcommand{\Cov}{\mathrm{Cov}}
\newcommand{\Var}{\mathrm{Var}}
\newcommand{\E}{\mathbf{E}}

\newcommand{\mbf}{\mathbf}
\newcommand{\mbb}{\mathbb}
\newcommand{\bs}{\boldsymbol}

\newcommand{\ta}{\theta}
\newcommand{\pv}{p(\mathbf{v})}
\newcommand{\pvh}{p(\mathbf{v},\mathbf{h})}

\newcommand{\Pv}{P(\mathbf{v})}
\newcommand{\Pvh}{P(\mathbf{v},\mathbf{h})}

\newcommand{\Qv}{Q(\mathbf{v})}
\newcommand{\Qta}{Q(\bs{\ta})}
\newcommand{\Pta}{P(\bs{\ta})}

\newcommand{\bphi}{\bs{\phi'}}

\newcommand{\cond}{\,\,\bigm\vert\,\,}
\newcommand{\simp}{\bs{\Delta}^{m-1}_{\alpha}}

\newcommand{\argv}{\argmax_{\mbf{v}}}
\newcommand{\argh}{\argmax_{\mbf{h}}}
\newcommand{\argvh}{\argmax_{\{\mbf{v},\mbf{h}\}}}
\newcommand{\vs}{\mbf{v}^{\bigstar}}
\newcommand{\vd}{\mbf{v}^{\blacklozenge}}

\newcommand{\erf}{\mathrm{erf}}
\newcommand{\erfc}{\mathrm{erfc}}

\newcommand{\sit}{\sigma_{\theta}}
\newcommand{\bsta}{\bs{\ta}}
\newcommand{\bsv}{\bs{\varphi}}

\begin{document}

\title{Mode-Assisted Unsupervised Learning of Restricted Boltzmann Machines}

\author{Haik Manukian}
\thanks{Equal contribution}
\email{email: hmanukia@ucsd.edu}
\affiliation{Department of Physics, University of California, San Diego, La Jolla, CA 92093}

\author{Yan Ru Pei}
\thanks{Equal contribution}
\email{email: yrpei@ucsd.edu}
\affiliation{Department of Physics, University of California, San Diego, La Jolla, CA 92093}

\author{Sean R.B. Bearden}
\email{email: sbearden@ucsd.edu}
\affiliation{Department of Physics, University of California, San Diego, La Jolla, CA 92093}

\author{Massimiliano Di Ventra}
\email{email: diventra@physics.ucsd.edu}
\affiliation{Department of Physics, University of California, San Diego, La Jolla, CA 92093}

\begin{abstract}
Restricted Boltzmann machines (RBMs) are a powerful class of generative models, but their training requires computing a gradient that, unlike supervised backpropagation on typical loss functions, is notoriously difficult even to approximate.
Here, we show that properly combining standard gradient updates with an {\it off-gradient} direction, constructed from samples of the RBM ground state (mode), improves their training dramatically over traditional gradient methods. This approach, which we call {\it mode training}, promotes faster training and stability, in addition to lower converged relative entropy (KL divergence).
Along with the proofs of stability and convergence of this method, we also demonstrate its efficacy on synthetic datasets where we can compute KL divergences exactly, as well as on a larger machine learning standard, MNIST. The mode training we suggest is quite versatile, as it can be applied in conjunction with any given gradient method, and is easily extended to more general energy-based neural network structures such as deep, convolutional and unrestricted Boltzmann machines.
\end{abstract}

\maketitle

Boltzmann machines~\cite{ackley1985learning} and their restricted version (RBMs), are generative models applied to a variety of machine learning problems~\cite{goodfellow2016deep}. 
They enjoy a universal approximation theorem for discrete probability distributions~\cite{le2008representational}, are used as building blocks for deep-belief networks~\cite{bengio2009learning} and, in no small feat, can even represent correlated states in quantum many-body systems~\cite{carleo2017solving,gao2017efficient}. 

Training RBMs is typically formulated as a gradient descent in the Kullbach-Leibler (KL) divergence between the data distribution defined by a dataset, and the RBM model distribution, parameterized by a set of weights and biases. This unsupervised procedure results in a computationally intractable expectation value popularly approximated by a Markov Chain Monte Carlo (MCMC) procedure dubbed ``contrastive divergence'' (CD)~\cite{hinton2002training}. This approach faces difficulty when the model distribution represented by the RBM contains peaks of probability far away from the elements of the dataset, resulting in ``spurious modes'' that trap the Markov chain~\cite{bengio2009learning}. The limitations of CD, the standard algorithm for training RBMs, combined with the rapid advances in supervised learning approaches, has led to the sideline of their unsupervised learning, known also as ``pretraining'', in favor of supervised backpropagation from random initial conditions~\cite{bengio2009learning}. 

However, many state-of-the-art neural networks have been shown to be vulnerable to what are called adversarial examples~\cite{szegedy2013intriguing}, or slight perturbations of the input that ``fool'' the network. In fact, one of the most popular supervised learning techniques, batch normalization, was found to contribute to this phenomenon~\cite{galloway2019batch}. 
It is known that pretraining can be a strong regularizer~\cite{erhan2010does} resulting in better generalization for supervised models, and an improvement in their unsupervised training could lead to more robust performance in a downstream task. This motivates the search for better unsupervised training methods. 

In a recent work, a memcomputing-assisted training scheme for RBMs~\cite{manukian2019accelerating} was proposed to address this unsupervised training difficulty. 
Memcomputing~\cite{13_memcomputing} is a novel computing paradigm in which memory (time non-locality) assists in the computation of hard computational decision and optimization problems~\cite{traversa2017polynomial, DMMperspective}. 
In the algorithm of Ref.~\onlinecite{manukian2019accelerating}, the difficult model expectation term in the log-likelihood gradient was replaced by a sample of the mode of the RBM's probability distribution obtained from a memcomputing solver, which led to better quality solutions versus CD in a downstream classification task. 
However, despite demonstrating a significant reduction in the number of pretraining iterations necessary to achieve a given level of classification accuracy, as well as a total performance gain over CD, the algorithm of Ref.~\onlinecite{manukian2019accelerating} does not fully exploit samples of the mode, in particular it does not give rise to training advantages over standard methods in the unsupervised setting. Additionally, in the present work, we introduce $i$) a principled schedule for incorporating samples of the RBM ground state into pre-training, $ii$) an appropriate mode-driven learning rate, $iii)$ comparisons to other state-of-the-art unsupervised pre-training approaches without the need of supervised fine-tuning, and $iv$) proofs of advantageous properties of the method.

We show that by appropriately combining the RBM's {\it mode} (ground state) samples and data initiated chains (as in CD) not only improves considerably the model quality over CD and other MCMC procedures, but also improves the {\it stability} of the pre-training routine. 
This {\it mode training} utilizes both the dataset (as in CD) and samples of the mode of the RBM model distribution in the training process to ``push down" spurious modes of the model, whenever they appear. 

Superficially, this method resembles `mode hopping' MCMC proposed in recent literature~\cite{sminchisescu2007generalized,lan2014wormhole}, where local maxima are either found with some optimization method or assumed to be known before hand (via a dataset). 
However, a crucial difference between our mode training for RBMs and mode hopping algorithms is that we do {\it not} use the modal configuration to initiate a new MCMC update to improve the mixing rate. Instead, the mode itself is used to inform the weight updates {\it directly}. The difference is substantial. In fact, since higher energy states are exponentially suppressed, exposing the Markov chain to the mode will most likely get it stuck there, which requires {\it ad hoc} constructions to recover detailed balance. Our mode-training method does not suffer from these drawbacks and is thus a more computationally efficient way to utilize the mode to train RBMs. Furthermore, we show that under a sufficiently large learning rate, sampling the global mode alone is capable of exploring efficiently a multi-modal energy landscape. 

To realize this method in practice, one must supplement standard gradient updates with updates constructed from samples of the ground state of the RBM. Finding this ground state is equivalent to a Quadratic Unconstrained Binary Optimization (QUBO) problem, known to be NP-hard~\cite{computational_complexity_book}. Therefore, although we can compute the ground state of RBMs exactly for small datasets, for efficient mode sampling in realistically sized cases, we employ a memcomputing solver that compares favorably to other state-of-the-art optimizers in efficiently sampling the ground states of similar non-convex landscapes~\cite{traversa2018evidence,sheldon2018taming}. The details of our implementation, including computational complexity and energy comparisons with MCMC, can be found in the appendix that accompanies this work. However, in principle, one could use other optimizers for mode training. 

To corroborate our method, we will show exact KL/log-likelihood achieved on small synthetic datasets and on the MNIST dataset. In all cases, we find that mode training is able to learn more accurate models than several other training methods such as CD, persistent contrastive divergence (PCD), parallel tempering (PT) used in tandem with enhanced gradient RBMs (E-RBMs), and centered RBMs (C-RBMs), as reported in Ref.~\onlinecite{melchior2016center}. 

The paper is organized as follows. In Section~\ref{Train-RBM}, we introduce RBMs and the standard unsupervised training procedures, and identify their main weaknesses. 
Section~\ref{Train-Mode} introduces our mode-training method and its main features. Section~\ref{Results} contains the results of our numerical experiments. 
Finally, we offer our thoughts for future work in Section~\ref{conclusions}.

\section{Training RBMs with MCMC} \label{Train-RBM}
RBMs are undirected graphical models with a bipartite structure that differentiates between $n$ visible nodes, ${\bf v} \in \{0, 1\}^n$ and a set of $m$ latent, or `hidden' nodes, ${\bf h} \in \{0, 1\}^m$, not directly constrained by the data~\cite{goodfellow2016deep}. 
These states are usually taken to be binary but can be generalized. 
Each state of the machine corresponds to an energy of the form
\begin{equation}
\label{eq:E}
E({\bf v}, {\bf h}) = - {\bf a}^T{\bf v} - {\bf b}^T{\bf h} - {\bf v}^T {\bf W} {\bf h},
\end{equation}
where the biases ${\bf a} \in \mathbb{R}^{n}$, ${\bf b} \in \mathbb{R}^{m}$, and weights ${\bf W} \in \mathbb{R}^{n\times m}$ are the learnable parameters. 
Note that an RBM does not allow connections {\it within} a layer. 
This defines a distribution over states given by a Boltzmann-Gibbs distribution, 
\begin{equation}
\label{eq:pvh}
p({\bf v}, {\bf h}) = \frac{e^{-E({\bf v}, {\bf h})}}{\mathcal{Z}}.
\end{equation}
The normalizing factor, $\mathcal{Z} = \sum_{\{{\bf v}\}} \sum_{\{{\bf h}\}} e^{-E({\bf v}, {\bf h})}$, is the formidable partition function that involves the sum over an exponentially scaling number of states, thus making the exact computation of its value infeasible. Additionally, the bipartite structure of the RBM connectivity implies that the hidden nodes are conditionally independent given any visible nodes (and vice versa), with a closed form conditional distribution given by~\cite{hinton2002training} $p(h_j = 1 | {\bf v}) = \sigma( \sum_i w_{ij} v_i + b_j)$, where $\sigma(x) = (1 + e^{-x})^{-1}$. 

We indicate the {\it unique} elements of the dataset for training and testing of the network as $\mathcal{D} = \{ {\bf v}_1, \cdots, {\bf v}_{n_d}\} \subset \Omega$, where $\Omega = \{0,1\}^n$ is the space of all binary sequences of length $n$. We can then write the data distribution as 
\begin{equation}
\label{eq:support}
q({\bf v}) = \sum_{{\bf v}_i \in \Omega} c_i \mathds{1}_{\mathcal{D}}({\bf v}_i),
\end{equation}
where $\mathds{1}$ is the indicator function that evaluates to $1$ if ${\bf v}_i \in \mathcal{D}$ and $0$ otherwise. This effectively defines a 
probability mass function (PMF) over $\Omega$ with non-zero values only for values ${\bf v}_i \in \mathcal{D}$. We then call $\mathcal{D}$ the {\it support} of $q$. 

Let us assume further that all data points have equal amplitude over the support, i.e.,  $c_i = 1/n_d$. Since most real world datasets consist of unique elements with no exact repeats, this class of distributions includes all relevant ones. However, we will see in Sec.~\ref{Results} that our method 
seems to work equally well also for non-uniform distributions. 

Training an RBM then amounts to a search for the appropriate weights and biases, $\theta = \{{\bf W}, {\bf a}, {\bf b}\}$, that will minimize the quantity
\begin{equation}
\label{eq:kl}
\text{KL}(q || p) = \sum_{\{{\bf v}\}} q({\bf v}) \log \frac{q({\bf v})}{p({\bf v})}.
\end{equation}

This is known as the Kullback-Leibler (KL) divergence between the data distribution, $q({\bf v})$, and the model distribution of the RBM over the {\it visible} layer,  $p({\bf v}) = \sum_{\{ {\bf h} \}} p({\bf v}, {\bf h})$, with hidden nodes traced out. The latter can be written as,
\begin{equation}
\label{eq:pv}
p({\bf v}) = \frac{1}{\mathcal{Z}} \prod_{i=1}^{n} e^{a_iv_i} \prod_{j=1}^{m}\left(1 + e^{b_j + \sum_{i' = 1}^{n} w_{i'j}v_{i'}}\right).
\end{equation}

The optimization of Eq.~(\ref{eq:kl}) is typically done with gradient descent over the KL divergence which leads to weight updates of the form~\cite{fischer2012introduction},
\begin{equation}
\label{eq:w_gradient}
\Delta w_{ij} \propto \left[ \langle v_i h_j \rangle_{q({\bf v})p({\bf h} | {\bf v})} - \langle v_i h_j \rangle_{p({\bf v}, {\bf h})} \right].
\end{equation}

The first term on the rhs of Eq.~(\ref{eq:w_gradient}) is an expectation with the hidden nodes driven by the data, and hence is referred to as the {\it data} term. Since the conditional distributions across the hidden nodes are factorial and known in closed form, this inference problem is easy in the RBM case. 
The second term on the rhs of Eq.~(\ref{eq:w_gradient}), instead, is an expectation over the model distribution with no nodes fixed, and called the {\it model} term. 
The exact calculation of this term requires computing the partition function of the RBM, which is proved to be hard even to estimate~\cite{long2010restricted}. It is this term that MCMC algorithms (including CD) attempt to approximate.

\subsection{Contrastive Divergence} 
A popular method to training RBMs is CD, which is a special case of an MCMC method known as Gibbs sampling~\cite{hinton2002training}. 
The Markov chain is initialized from a point in the dataset, ${\bf v}$, then the hidden and visible nodes are sequentially re-sampled a $k$ number of times. A distorted model expectation is then computed from the {\it reconstructed} ${\bf v}^k$. 
In practice, choosing some finite $k$ introduces a bias, but empirically it is found that using $k=1$ gives a sufficient signal for learning~\cite{carreira2005contrastive}. 
Since the CD chain starts from a point in the dataset (i.e., a sample from the data distribution), difficulties arise when the model distribution represented by the RBM contains modes at points where the data distribution has negligible probability. 
CD will have a hard time finding and hence pushing down these spurious modes. This, coupled with the prohibitively slow mixing of this MCMC method due to random-walk exploration, and typical high dimensionality of the problem, renders CD not a particularly effective method for unsupervised learning. 

\section{Mode Informed Weight Updates}\label{Train-Mode}

After these preliminaries we can now describe our mode-training method. In a nutshell, it consists in replacing the average in the model term of  Eq.~(\ref{eq:w_gradient}) with the {\it mode} of $p({\bf v})$ at appropriate steps of the training procedure. However, $p({\bf v})$ is very cumbersome to compute (see Eq.~(\ref{eq:pv})), thereby adding a considerable computational burden. Instead, we sample the {\it mode} of $p({\bf v},{\bf h})$, the model distribution of the RBM. 

The rationale for replacing the mode, ${\bf v}^+$, of $p({\bf v})$ with the visible configuration of the mode, ${\bf v}^*$, of $p({\bf v},{\bf h})$ is 
because the two modes are {\it equivalent} with high probability under scenarios typical for different stages of the RBM pre-training. We prove this 
{\it rigorously} in the appendix, while here we provide numerical evidence of this fact. 

\subsection{Mode Correspondence of Joint and Marginal PMF}\label{modeeq}

\begin{figure}[t!]
	\centering
	\includegraphics[width=0.48\textwidth]{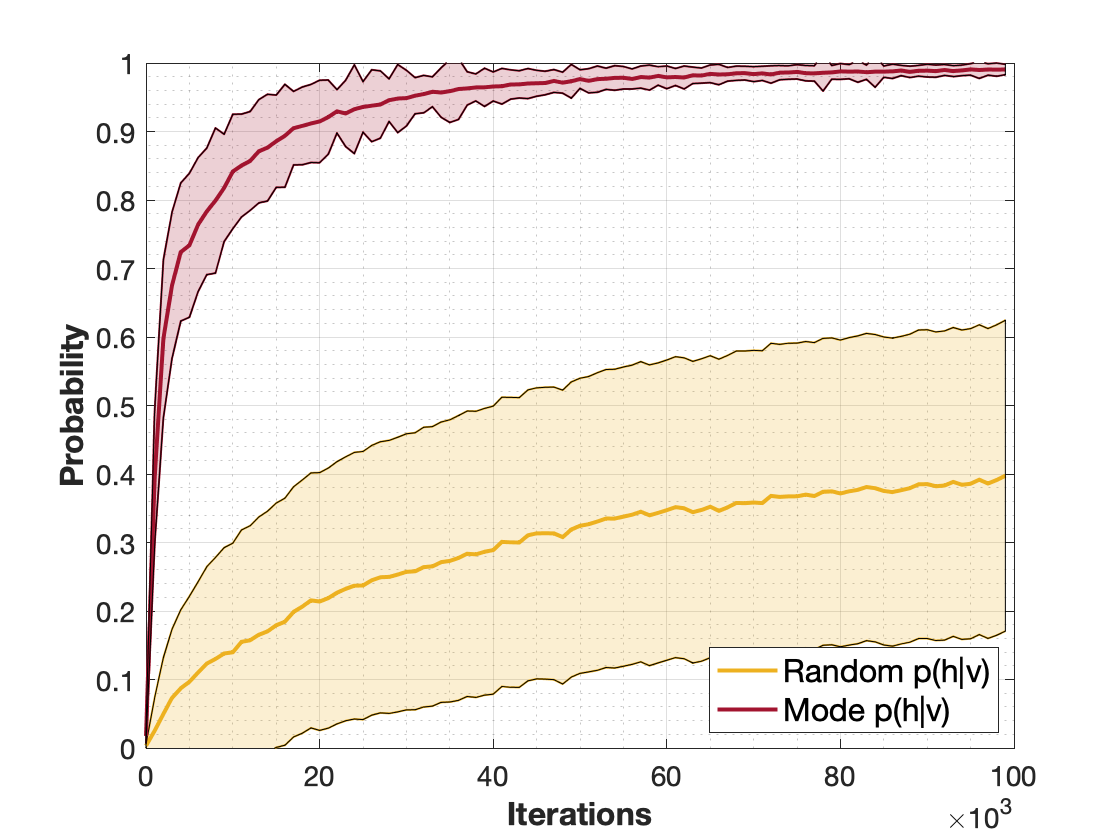}
	\caption{The maximal conditional probability distribution of the hidden layer, $r({\bf v}^+)$, when driven by ${\bf v}^+$, the mode of the marginal PMF, $\pv$, as a function of CD-1 training iterations.  The results are averaged over an ensemble of 200 randomly initialized $15\times 10$ RBMs, with $\pm 1\sigma$ error bars (the shaded regions), on a shifting bar synthetic dataset. The same calculation conditioned on a random visible configuration is plotted as a baseline for comparison.}
	\label{fig:mode_equiv}
\end{figure}

To illustrate the equivalence between the modes of $\pv$ and $\pvh$, let us begin by expressing the joint PMF in terms of the product of the marginal PMF over the visible layer and the conditional PMF over the hidden layer $\pvh = \pv \phvc$. For any given visible configuration ${\bf v}$, we then have $\argh \pvh = \pv \left[\maxh \phvc \right]$. We can then define the hidden ``activation'' of ${\bf v}$ to be
\begin{equation}
r({\bf v}) = \maxh \phvc,
\end{equation}
which allows us to write $\maxh\pvh = \pv\rv$. Note that we can interpret $\rv$ as a measure of the ``certainty" that the hidden nodes acquire the value $0$ or $1$. 

It is then clear that we can write the probability amplitude of the mode of the joint PMF as $\maxvh \pvh = \maxv (\maxh \pvh) = \maxv(\pv\rv) \leq \maxv(\pv) = p({\bf v}^+)$, where we have used the fact that $\rv \leq 1$ and ${\bf v}^+$ is the mode of the marginal PMF, $\pv$. If $r({\bf v}^+)=1$ then 
we have modal equivalence of the joint and marginal PMFs. 

In Fig.~\ref{fig:mode_equiv}, we plot the evolution of $r({\bf v}^+)$ as a function of the number of CD-1 training iterations for a shifting bar synthetic 
dataset, which is small enough that we can compute the exact mode of $\pv$ at any iteration. The figure indeed shows that $r({\bf v}^+)$ approaches $1$ rather quickly as pre-training proceeds. The activation of a random visible configuration is being used as comparison. 

In the appendix we also prove that the condition of $r({\bf v}^+)$ being close to 1 is not necessary for establishing modal equivalence. In fact, we prove that it is still possible for the two modes to be equal even when the weights are small (thus a smaller $r({\bf v}^+)$ value). Additionally, we show in the 
appendix that mode training is more effective in exploring the PMF of the model distribution for RBM instances of greater frustration. The latter is a measure of the degeneracy of the low-energy states of an RBM, and thus the difficulty of finding the ground state configuration. Since it was shown that the frustration of the RBM increases as pre-training proceeds~\cite{rbm_loops}, in order to effectively utilize the power of mode training, the frequency of mode updates should be higher at the later stages of the training than the earlier stages.

\subsection{Optimal Mode-Training Schedule}

The results from the previous subsection then suggest a schedule for the mode training routine that performs mode updates more frequently the longer the pre-training routine has elapsed. 

To realize this, we use a sigmoid, $\sigma$, to calculate the probability of replacing the data driven hidden CD term with a mode driven term at the iteration step $n$
\begin{equation}
\label{eq:p_mode}
P_{\text{mode}}(n) = P_{\text{max}}\sigma(\alpha n + \beta).
\end{equation}
Here, $0 < P_{\text{max}}\leq 1$ is the maximum probability of employing a mode update, and $\alpha$ and $\beta$ are parameters that control how the mode updates are introduced into the pre-training. They are chosen such that the frequency of mode updates becomes dominant only when both the conditions of large weights and frustration are met (see Sec.~\ref{Results} for the value of these parameters). Initially, $P_\text{mode}$ will be small, since the joint- and marginal-distribution modes are unequal, and gradually rises to its maximal value when the modes are of equal magnitude. 
Note that one may employ different functions to quantify the degree to which the joint- and marginal-distribution modes equalize during training. However, we have found that the sigmoid works well enough in practice.

\subsection{Combining MCMC with mode updates}

We are now ready to outline the full procedure of mode training, that combines a MCMC method with the mode updates following the schedule~(\ref{eq:p_mode}). Although one may choose any variation of the MCMC method to train RBMs, for definiteness of discussion, we consider here the standard training method, CD~\cite{hinton2002training}.  In this case, weight updates follow the modified KL$(q||p)$ gradient. As discussed in Section~\ref{Train-RBM}, it evaluates to a difference of two expectations called the data term and model term which we can write as 

\begin{equation}
\label{eq:cd_update}
\Delta w^{\text{CD}}_{ij} = \epsilon^\text{CD}\left[\langle v_i h_j \rangle_{q({\bf v}) p({\bf h} | {\bf v})} - \langle v_i h_j \rangle_{p^k({\bf v}, {\bf h})}\right],
\end{equation}
where $\epsilon^\text{CD}$ is the CD update learning rate, and the expectation in the second term is taken over the reconstructed distribution over a Markov chain initialized from the dataset after $k$ Gibbs samples ($k=1$ in most cases).
When driving the weights with samples of the RBM ground state with the schedule~(\ref{eq:p_mode}), we use instead the following update,
\begin{equation}
\label{eq:mode_update}
\Delta w^{\text{mode}}_{ij} = \epsilon^\text{mode}\left[\langle v_i h_j \rangle_{q({\bf v}) p({\bf h} | {\bf v})} -  [v_i h_j]_{p({\bf v}, {\bf h})} \right],
\end{equation}
where $[\,]_p$ is the mode of the joint RBM model distribution. Note that the mode update learning rate, $\epsilon^\text{mode}$, may be different from the CD learning rate, $\epsilon^\text{CD}$. 

We also stress that the updates in Eq.~(\ref{eq:mode_update}) are in an {\it off-gradient} direction. As we show now, this is the reason for the increased stability of the training over MCMC approaches, and its convergence to arbitrarily small KL divergences. 

\subsection{Stability and Convergence}

The data term, which is identical in both Eq.~(\ref{eq:cd_update}) and Eq.~(\ref{eq:mode_update}), tends to {\it increase} the weights associated with the visible node configurations in the dataset, thereby increasing their relative probabilities compared to states not in the support set, ${\bf v} \in \Omega \backslash \mathcal{D}$.
Instead, the model term {\it decreases} the weights/probability corresponding to highly-probable model states.
CD does this poorly and often diverges, while mode training achieves this with better stability and faster convergence (see Fig~\ref{fig:mode-vs-cd}). We provide here an intuitive explanation of this phenomenon, while a formal treatment on this topic will be provided in the appendix. 
\begin{center}
	\begin{figure*}[t]
		\centering
		\includegraphics[width=0.8\textwidth]{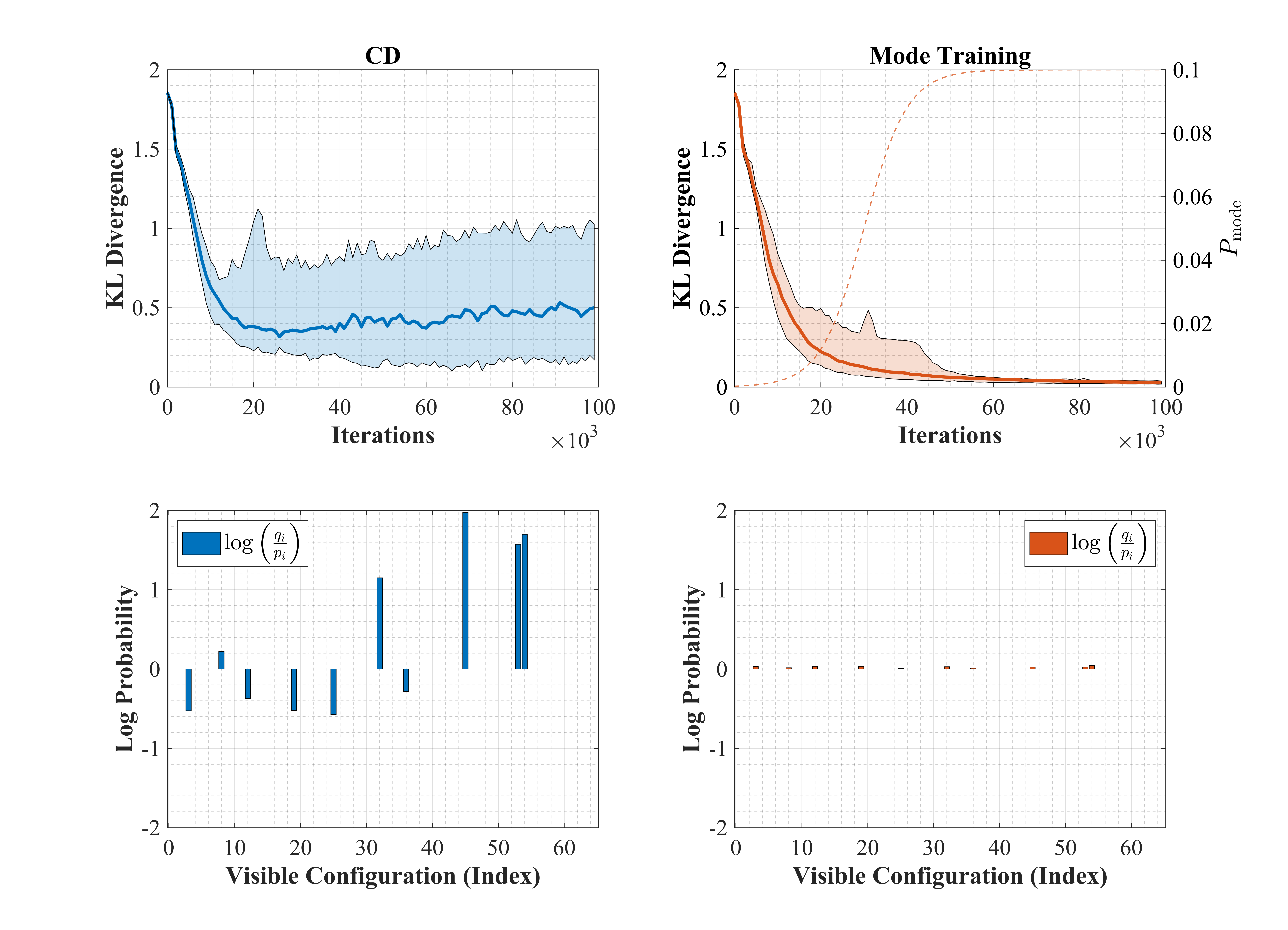}
		\caption{Median KL divergences +max/-min KL (top row) and converged logarithmic differences between model and data distributions (bottom row) of 25 randomly generated $6 \times 6$ RBM on a random uniform support set of size $n_d = 10$ for CD-1 (left column) and mode training (right column). In both cases, the learning rate was a constant $\epsilon^\text{CD} = 0.05$ for 100,000 iterations. The mode sampling probability, $P_\text{mode}$, is plotted as the dotted line in the top right.}\label{fig:mode-vs-cd}
	\end{figure*}
\end{center}

The pre-training routine can be broken down in two phases. In the first phase, the training procedure attempts to discover the support $\mathcal{D}$ of the data distribution $q({\bf v})$. We call this phase the {\it discovery} phase. To better see this, consider a randomly initialized RBM with small weights. These small and uncorrelated weights give rise to RBM energies close to zero for all nodal states, or $E({\bf v},{\bf h}) \approx 0$ for all ${\bf v}$ and ${\bf h}$, see Eq.~(\ref{eq:E}). This results in the model distribution $p({\bf v},{\bf h})$ being almost uniform. 

Therefore, we see that in the discovery phase of training, the model term plays little role in the training as it simply pushes down on the weights in a practically uniform manner, with $\langle v_i h_j \rangle_\text{M} \approx 0.25$. On the other hand, the data term drives the initial phase of the training by increasing the marginal probability of the visible states in the support, ${\bf v} \in \mathcal{D}$. We can then employ a large learning rate (say, $\epsilon^{\text{CD}} = 1$) in the beginning of the training, driving the visible layer configurations in the dataset, $\mathcal{D}$, to high probability versus configurations outside the support. Empirically, we find that CD training performs in the discovery phase reasonably well, and is quickly able to ``find" the visible states in the support.

Now, having discovered the support, we arrive at the second phase of the training where we have to bring the model distribution as close to uniformity as possible over the support in order to minimize the KL-divergence. We call this phase the {\it matching} phase of the training, where we bring the model distribution as close to the data distribution as possible. CD usually performs poorly in this phase (see Fig.~\ref{fig:mode-vs-cd}). To see this most directly, we simply have to consider a visible state with a slightly larger probability than the other states. It should then be necessary for the model term to locate and ``push down" on this state to increase the uniformity of the distribution over the support. However, for any CD approximation of the model term, this rarely happens in a timely manner as the mixing rate of the MCMC chain is far too slow to locate this state before the training diverges. 

This is where samples of the mode are most effective, and can assist in the correction of the states' amplitudes. As we have discussed in Sec.~\ref{modeeq}, finding the modal state, ${\bf v}^*$, of the model distribution, $\pvh$ allows us to immediately locate the mode, ${\bf v}^+$, of the marginal probability, $\pv$, and ``push" down on this state through an iteration of weight updates. This ``push" may result in another state ``popping" up and becoming the new modal state; however, often times the probability amplitude of this new state will be less than that of the previous mode (see also the appendix). This results in a training routine that ``cycles" through any dominant state that emerges at each iteration, and the probability amplitude of the mode decreases as training proceeds until the probability amplitudes of all the states in the support become equal (see the formal demonstration of this in the appendix), which results in the desired uniform distribution over the support. This can be visualized as a ``seesaw" between the dominant states, with the oscillation amplitude of this seesaw decaying to zero in time.

We outline the pseudo-code for mode training in~Algorithm \ref{mode_algorithm} and a visual depiction of the training side by side with CD-1 is shown in Fig.~\ref{fig:mode-vs-cd}. 

As it should now be clearer, these mode-driven updates are {\it deviations} from the gradient direction, since in general the mode over the model distribution is different from the expected value. This makes the mode-training algorithm, which mixes mode driven samples and data-driven ones, {\it distinct} from gradient descent. This is also supported by the fact that our method tends toward a particular class of distributions (uniform), when gradient descent would settle in some local minima or saddle points in the KL landscape.

\begin{algorithm}[H]
	\caption{Unsupervised learning of an RBM with the mode-training algorithm}\label{mode_algorithm}
	\begin{algorithmic}[1]
		\Procedure{MT}{$P_{\text{max}}, \alpha, \beta, \{\epsilon^{\text{CD}}_n\}_{n=1}^N, N$}
		\State $\theta_0 \sim \mathcal{N}(0,0.01)$
		\For{$i = 1; i \leq N; i$++}
		\State $p_\text{mode} \gets P_\text{max}\sigma(\alpha i + \beta)$
		\State Sample $u \sim \text{U}[0,1]$
		\If{$u \leq p_\text{mode}$}
		\State ${\bf v}^*, {\bf h}^*, E_0 \gets \text{argmin} E({\bf v}, {\bf h})$
		\State $\gamma \gets \frac{-E_0}{(n + 1)(m + 1)}$
		\State $\theta_i \gets \theta_{i-1}  + \gamma \epsilon^\text{CD}_i \Delta \theta^\text{mode}$  \Comment{Eq.~\ref{eq:mode_update}}
		\Else
		\State $\theta_i \gets \theta_{i-1} + \epsilon^\text{CD}_i \Delta \theta^\text{CD}$ \Comment{Eq.~\ref{eq:cd_update}}
		\EndIf
		\EndFor
		\State \textbf{return} $\theta_N$
		\EndProcedure
	\end{algorithmic}
\end{algorithm}

The free parameters in this method are the schedules of the mode sample using $P_\text{mode}(n)$ (defined by $P_{\text{max}}$, $\alpha$ and $\beta$ 
in Eq.~(\ref{eq:p_mode})) and the CD learning rate, $\epsilon^\text{CD}$. With $\epsilon^\text{CD}$ fixed, we set $\epsilon^\text{mode} = \gamma \epsilon^\text{CD}$, where $\gamma = -E_0/[(n + 1)(m + 1)]$, with $E_0 (<0)$ being the ground state of the corresponding RBM with nodal values $\{-1,1\}^{n+m}$. 
This particular choice of $\gamma$ is an upper bound to the learning rate which minimizes the RBM energy variance over all states (see the appendix for the proof of this statement). 

Finally, we find that the mode training method is not very sensitive to the parameters chosen. In fact, as long as the mode samples are incorporated after the joint and marginal mode equilibration, the training is stabilized and the learned distribution will tend to uniformity (see also the appendix). This result reinforces the intuitive notion that the pushes on the mode provide a stabilizing quality to the training over CD (or any other MCMC approach), which can otherwise diverge when mixing rates grow too large at later times during training. 

\subsection{Importance of Representability}
Note that since mode training is driven to distributions of a particular form, instead of local minima as in the case of CD or other gradient approaches, the representability of the RBM becomes important. The ability of a RBM to represent a given data distribution is given by the amount of hidden nodes, where one is guaranteed universal representability with $n_d +1$ hidden nodes~\cite{le2008representational}. In other words, one more hidden node than the number of visible configurations with non-zero probability is sufficient (but perhaps not necessary) for general representability. In practice, this bound is found to be very conservative and typically much fewer nodes are needed for a reasonable solution. 

Representability can become an issue in mode training when the parameter space of the RBM does not include the uniform distribution over the support (or a reasonable approximation). Since the mode training is generally in a non-gradient direction, this means that it may settle to a worse solution than a local optimum obtainable by CD. This is a signal that more hidden nodes are required for an optimal solution. 

Since most natural datasets live on a very small dimensional manifold of the total visible phase space, $|n_d|/|\Omega| \ll 1$, the amount of hidden nodes required typically scales polynomially with the size of the problem, versus the exponential scaling of the visible phase space. This makes representability not an insurmountable problem for mode training, even in full size problems. To this end, the examples of Fig.~\ref{fig:mode-vs-cd} and Fig.~\ref{fig:synthetic_results} show that mode training does not necessarily fail if the number of hidden nodes is less than that needed to guarantee representability. 

\section{Results}\label{Results}
As examples of our method, we have computed the log-likelihoods achieved with mode training across two synthetic and one realistic (MNIST) datasets, and compared the results against the {\it best} achieved log-likelihoods with CD-1, PCD-1 and PT on standard RBMs, E-RBMs, and C-RBMs ~\cite{melchior2016center}. 
For the small synthetic datasets we could also compute the {\it exact} log-likelihoods, thus providing an even stronger comparison. For the larger MNIST case, mode sampling was done via simulation of a digital memcomputing machine based on Ref.~\onlinecite{bearden2019critical}. The specific details of our implementation can be found in the appendix. 

For synthetic data, we use the commonly employed {\it binary shifting bar} and {\it bars and stripes} datasets, both described in Ref.~\onlinecite{mackay2003information}. The former is defined by two parameters: the total length of the vector, $L$, and the amount of consecutive elements (with periodic boundary conditions), $B < L$, set to one, with the rest set to zero. This results in $L$ unique elements in the dataset with uniform probability, giving a maximum likelihood of $L \log(1/L)$. The {\it inverted} shifting bar set is obtained by swapping ones and zeros. The bars and stripes dataset is constructed by setting each row of a $D\times D$ binary pattern to one with probability 1/2, and then rotating the resulting pattern $90^\circ$ with probability 1/2. This produces $2^{D+1}$ elements, with the all-zero and all-one patterns being twice as likely as the others.

For a direct comparison to previous work, we followed the same setup as Ref.~\onlinecite{melchior2016center}. A 9$\times$4 RBM was tested on a shifting bar 
dataset with $L = 9$, $B = 1$ and a $D = 3$ bars and stripes dataset. Both synthetic sets were trained for 50,000 parameter updates, with no mini-batching, and a constant $\epsilon^{\text{CD}} = 0.2$. For the MNIST dataset, a $784 \times 16$ sized model was trained for 100 epochs, with batch sizes of 100. The mode samples in both cases are slowly incorporated into training in a probabilistic way following Eq.~(\ref{eq:p_mode}), initially with $P_{\textrm{mode}} = 0$ and driven to $P_{\textrm{max}} = 0.1$ for the shifting bar and MNIST datasets, and $P_{\textrm{max}} = 0.05$ for the bars and stripes dataset. In both cases, we chose $\alpha = 20/N$ and $\beta = -6$, where $N$ is the total number of parameter updates. 

\begin{figure}[h]
  \centering
    \includegraphics[width=0.4\textwidth]{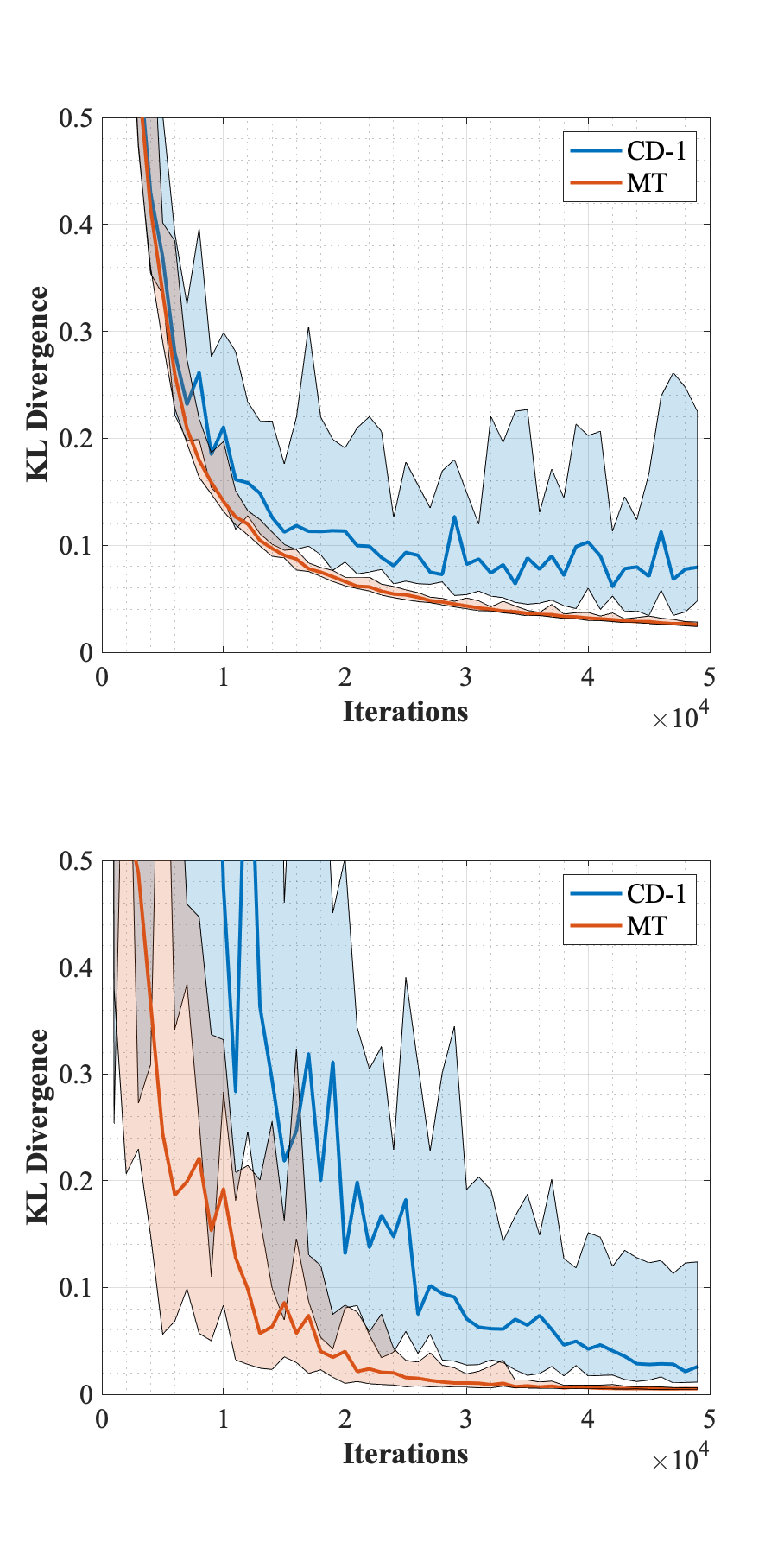}
      \caption{KL divergences achieved on the binary shifting bar dataset across 25 randomly initialized 14$\times$10 RBMs for both CD-1 and mode training (MT). In addition, every time a mode sample is taken, CD is allowed to run with $k = 720$, a number scaled to the equivalent computational cost of taking a mode sample (see text and appendix). The bold line represents the median KL divergence across the runs, and the max/min KL divergences achieved at that training iteration define the shaded area. The plot in the top panel is with a small CD learning rate, $\epsilon^{\textrm CD} = 0.05$. The plot in the bottom panel is with an exponentially decaying $\epsilon^{\text{CD}}(n) = e^{-cn}$ with $c=4$ and $n\in [0,1]$ being the fraction of completed training iterations.}
\label{fig:synthetic_results}
\end{figure}

We plot an example of training progress in a moderately large synthetic problem in Fig.~\ref{fig:synthetic_results}. Reported is the KL divergence (which differs from the log-likelihood by a constant factor independent of the RBM parameters~\cite{goodfellow2016deep}) of a slightly bigger $14 \times 10$ RBM as a function of number of parameter updates on a $L=14$, $B = 7$ shifting bar set, for both CD-1 and mode training. We consider two learning rate schedules, constant $(\epsilon^\text{CD} = 0.05)$ and exponential decay $(\epsilon^\text{CD}(n) = e^{-cn}, c = 4, n \in [0,1] \text{, the fraction of completed training iterations})$. Additionally, every time a mode sample is taken, CD is allowed to run with $k = 720$, a number scaled to the equivalent computational cost of taking a mode sample. The details of the computational equivalence between a mode sample using memcomputing and iterations of CD are discussed in greater detail in the appendix. 
In both cases, even when computational cost is factored in, mode training leads to better solutions and proceeds in a much more stable way across runs (lower KL variance at convergence). Importantly, mode training {\it never} diverges while CD oftentimes does. Following our intuition about mode training established in Sec.~\ref{Train-Mode}, using larger learning rates in the CD-dominated phase accelerates the convergence of mode training.  

It is known that using CD to train RBMs can result in poor models for the data distribution ~\cite{hinton2010practical}, for which PCD and PT are recommended. We note that for the mode training employed in this paper, CD-1 was employed as the gradient approximation (except in the case for MNIST where PCD-1 was used). Impressively, in all cases tested, the mode samples were able to stabilize the CD algorithm sufficiently to overcome the other, more involved approximations (PT) and model enhancements (centering).

\begin{table}[ht!]
\begin{tabular}{|c||c|c|c|c||}
\hline
 & S. Bar &Inv. S. Bar & Bars \& Stripes & MNIST\\ \hline \hline
 CD-1& -20.42 & -20.73&-61.08 & -152.42\\ \hline
PCD-1& -21.71 & -21.64& -57.01 & -140.43 \\ \hline
 PT &   -20.57& -20.57& -51.99& -142.00\\ \hline
 $\text{MT}$ & {\bf -19.85} & {\bf-19.86} & {\bf -50.79}({\bf -41.82}) & -{\bf 136.42}\\ \hline
 Exact & -19.77 & -19.77 & -41.59 & -- \\ \hline
\end{tabular}
\caption{Comparison between the best log-likelihoods achieved over 50,000 gradient updates on a $9 \times 4$ RBM across various RBM types (standard, E-RBM, C-RBM) and training techniques (CD-1, PCD-1, PT) as reported in Ref.~\onlinecite{melchior2016center} compared with mode training (MT) on a standard RBM. For each technique, the best achieved log-likelihood score across 25 runs is reported. In parenthesis are results for a $9\times 9$ RBM. For these small 
	datasets we can also compare with the exact result. For MNIST, networks trained had 16 hidden nodes and PCD-1 was used as the gradient update, and average log-likelihood is reported.}\label{tablemel2016}
\end{table}

In addition, it is clear that mode training exhibits several desirable properties over CD (or other gradient approaches). Most significantly, it seems to perform better with larger learning rates during the gradient dominated phase, and smaller learning rates when using mode samples. CD and other gradient methods generally perform better with smaller learning rates, as their approximation to the exact gradient gets better. Irrespective, even in this regime, the mode training eventually drives the system to the uniform solution compared to the local optimum of CD. The main advantage is that with mode training, one can (and often should) use larger learning rates, resulting in fewer required iterations. 

For further comparison, we report in Table~\ref{tablemel2016} results for the shifting and inverted bar, bars and stripes, and MNIST datasets obtained with mode training 
and those reported in Ref.~\onlinecite{melchior2016center}. The results show mode training with a standard RBM always converges to models with log-likelihoods higher than E-RBMs, and C-RBMs trained with CD-1, PCD-1, or PT. Furthermore, the mode training log-likelihood increases with an increasing number of hidden nodes (better representability). Empirically, we also find the incredible result that with sufficient representability and the proper learning rate, mode training can find solutions {\it arbitrarily close} to the exact distribution. 

\section{Conclusions}\label{conclusions}
In this paper we have introduced an {\it unsupervised non-gradient} training method that stabilizes gradient based methods by utilizing samples of the ground state of the RBM, and is empirically seen to get as close as desired to a given uniform data distribution. It relies on the realization that as training proceeds, the RBM becomes increasingly frustrated, leading to the modes of the visible layer distribution and joint model distribution becoming effectively equal. As a consequence, by using the mode (or ground state) of the RBM during training, our approach is able to ``flatten'' all modes in the model distribution that do not correspond to modes in the data distribution, reaching a steady state only when all modes are of equal magnitude. In this sense, the ground state of the RBM can be thought of as `supervising' the gradient approximation during training, preventing any pathological evolution of the model distribution. 

Our results are valid if the representability of the RBM is enough to include good approximations of the data distribution. Once the representability is sufficient, a properly annealed learning-rate schedule will take the KL divergence as low as desired. Increasing the number of hidden nodes increases the non-convexity of the KL-divergence landscape, easily trapping standard algorithms in sub-optimal states. In practice, after some point, increasing the number of hidden nodes will not decrease the KL divergence that a pre-training procedure actually converges to, as the trade-off between effective gradient update and representation quality is reached. We here claim that this point of tradeoff for our mode-assisted procedure is reached at far greater number of nodes than standard procedures, thus allowing us to find representations with far smaller KL-divergence. The mode training we suggest then provides an extremely powerful tool for unsupervised learning, which {\it i)} prevents a divergence of the model, {\it ii)} promotes a more stable learning, and {\it iii)} for data distributions uniform across some support, can lead to solutions with arbitrary high quality. 

To scale our approach, one would need an efficient way to sample low-energy configurations of the RBM, a difficult optimization problem on its own. This is equivalent to a weighted MAX-SAT problem, for which there are several industrial-scale solvers available. Also, the recent 
successes of memcomputing on these kind of energy landscapes in large cases (million of variables) are fodder for optimism~\cite{traversa2018evidence,sheldon2018taming}. 

Finally, fitting general discrete distributions (with modes of different height) with this technique 
seems also within reach. In this respect, we can point to our results on the bars and stripes dataset (a non-uniform $q({\bf v})$) for inspiration. 
We have found the best log-likelihood on that set with mode training with a {\it lower} frequency of the mode sampling, $P_\text{max}= 0.1 \to 0.05$, compared to the shifting bar (a uniform set). This suggests that a general update, which properly weighs the mode sample in combination with the dataset samples, may extend this technique to general non-uniform probabilities, with the weight analogous to a tunable demand for uniformity.

Our method is useful from a number of perspectives. First, from the unsupervised learning point of view, it opens the door to the training of RBMs with unprecedented accuracy in a novel, {\it non-gradient} approach. Second, many unsupervised models are used as `feature learners' in a downstream supervised training task (e.g., classification), where the unsupervised learning is referred to as pre-training. We suspect that starting backpropagation from an initial condition obtained through mode training would be highly advantageous. Third, the mode training we suggest can be done on models with any kind of pairwise connectivity, which include deep, convolutional, and fully-connected Boltzmann machines. We leave the analysis of these types of networks for future work. \\
 

\noindent {\bf Acknowledgments}\\
Work supported by DARPA under grant No. HR00111990069. H.M. acknowledges support from a DoD-SMART fellowship. M.D. and Y.R.P. acknowledge partial 
support from the Center for Memory and Recording Research at the University of California, San Diego. All memcomputing simulations reported in this paper have been done on a single core of an AMD EPYC server. \\




\bibliographystyle{apsrev4-1}
\bibliography{SUSYref}

\onecolumngrid
\newpage

\appendix
\title{Supplementary Material: Mode-Assisted Unsupervised Learning of Restricted Boltzmann Machines}

\author{Haik Manukian}
\thanks{Equal contribution.}
\affiliation{Department of Physics, University of California, San Diego, La Jolla, CA 92093}

\author{Yan Ru Pei}
\thanks{Equal contribution.}
\affiliation{Department of Physics, University of California, San Diego, La Jolla, CA 92093}

\author{Sean R.B. Bearden}
\affiliation{Department of Physics, University of California, San Diego, La Jolla, CA 92093}

\author{Massimiliano Di Ventra}
\affiliation{Department of Physics, University of California, San Diego, La Jolla, CA 92093}

\maketitle

\section{Sampling with Memcomputing}
\label{mem}

The mode-training method introduced in the main text requires sampling the mode of the model distribution of a given RBM. This task can be transformed to sampling the optimum of an equivalent weighted, mixed maximum satisfiability (MAX-2-SAT) optimization problem~\cite{manukian2019accelerating}.  
To obtain high-quality samples for large models, we employ the memcomputing approach~\cite{13_memcomputing,DMM2,DMMperspective}, a novel computing paradigm that employs memory to both store and process information. 

\subsection{Memory Dynamics}

Our implementation is based on the approach used in Ref.~\onlinecite{bearden2019critical} for the satisfiability (SAT) problem, appropriately modified for the MAX-2-SAT optimization problem. For a MAX-2-SAT with $N$ variables, $M_1$ 1-SAT clauses, and $M_2$ 2-SAT clauses we have $i \in [[1,N]]$ and $m \in [[1,M_2]]$. In this case, the equations used to simulate a digital memcomputing machine read 
\begin{align}
\label{eq:voltages}
\dot{v}_i&= b_i + \sum_m \Big\{ W_{2,m} x^f_m x^s_m G_m^i + \rho(1-x^f_m)R_m^i \Big\} \\
\label{eq:xfast}
\dot{x}^f_m &= \beta (x^f_m + \epsilon)(C_m - \frac{1}{4}),\\
\label{eq:xslow}
\dot{x}^s_m &= \alpha (1+W_{2,m}) C_m.
\end{align}
The voltages, $v_i \in [-1,1]$, are continuous representations of the $N$ Boolean variables of the problem, $y_i$, with a false assignment represented as $v_i < 0$, a true assignment represented as $v_i > 0$, and $v_i = 0$ is ambiguous. Rather than thresholding the voltages to check the clause states, we use the clause function directly. A 2-SAT clause in Boolean form is comprised of two literals, $\{l_{i,m}, l_{j,m}\}$, where a literal in the $m$-th clause, $l_{i,m}$, is either a negated, $\bar{y_i}$, or unnegated, $y_i$, variable. The Boolean clause is represented as a continuous clause function, 
\begin{equation}
C_m(v_i, v_j) = \frac{1}{2}\textrm{min}[(1 - q_{m,i} v_i),(1 - q_{m,j} v_j)].
\end{equation}
The factor $q_{m,i}$ contains the information about the relation between the literal in the $m$-th clause, $l_{i,m}$, and its associated variable, $y_i$; it evaluates to $+1$ if $l_{i,m}=v_i$, and $-1$ if $l_{i,m}=\bar{v_i}$. The function is bounded, $C_m \in [0,1]$, and we consider a clause to be satisfied when $C_m(v_i, v_j) < 0.5$. By thresholding the clause function we also avoid the ambiguity associated with $v_i = 0$.

Each clause has a ``fast'', $x^f_m$, and a ``slow'', $x^s_m$, memory variable that serve as indicators of the history of the state of $C_m(v_i, v_j)$. The memory is ``fast'' in the sense that it contains information of the {\it recent} history of $C_m$, and ``slow'' in the sense that it contains information on the {\it entire} history of $C_m$. Both memory variables are bounded,  $x^f_m \in [0, 1]$ and $x^s_m \in [1, 10*M_2]$. The offset $\epsilon=10^{-3}$ in Eq.~(\ref{eq:xfast}) is used to remove spurious steady-state solutions. \\

The gradient-like term in Eq.~(\ref{eq:voltages}) is $G_m^i=0$ if variable $y_i$ is not associated with any literal in clause $m$. Otherwise,
\begin{equation}
	G_m^i = q_{m,i} \frac{1}{2}(1 - q_{m,j}v_j),
\end{equation}
where $v_j$ is the value of the voltage corresponding to the other literal in the clause. The ``rigidity'' term in Eq.~(\ref{eq:voltages}) is
\begin{equation}
R_m^i = \\
\begin{cases}
q_{m,i} \frac{1}{2}(1 - q_{m,i} v_i), & C_m(v_i,v_j)=\frac{1}{2}(1 - q_{m,i} v_i)\\
0, & C_m(v_i,v_j) = \frac{1}{2}(1 - q_{m,j} v_j).
\end{cases}
\end{equation}
This term only influences the voltage that is closest to the satisfying assignment in the clause. \\

The weight of each 2-SAT clause, $W_{2,m}$, is incorporated in the dynamics of the slow memory variable and the dynamics of voltages. The weights of the 1-SAT clauses are used to bias the voltage dynamics in Eq.~(\ref{eq:voltages}) as $b_i = (W_{1,i} - W_{1,\bar{i}})/2$, where $W_{1,i}$ is the weight of the 1-SAT with a literal that is equivalent to variable $y_i$ and $W_{1,\bar{i}}$ is the weight of the 1-SAT with a literal that is the negation of variable $y_i$. The weight is zero if no corresponding 1-SAT exists. \\
 
The parameter values used for the simulations reported in the main text are $\alpha = 10$, $\beta = 0.1$, $\rho = 0.1$. At $t=0$, voltages are randomly initialized with $x^f_m=0$ and $x^s_m=1+W_{2,m}$. The equations are then numerically integrated with the forward Euler method using an adaptive time step, $\Delta t \in [2^{-5}, 2^{-1}]$, until a total integration time of $t = 500$ is reached. Then, we take the configuration with the lowest number of unsatisfied clauses as the sample.

\subsection{Computational Cost of Sampling with Dynamics}

For simplicity, let us assume the form of an $N\times N$ RBM, resulting in $2N$ voltage variables and $O(N^2)$ clauses as a MAX-2-SAT instance. The time 
complexity of a forward Euler integration step is dominated by the sparse matrix-vector multiplication of a $2N \times N^2$ sparse matrix. Since each node  connects to $N$ other nodes this matrix contains $N^2$ non-zero elements encoding the connectivity structure of the problem. This matrix multiplies the gradient vector of length $N^2$, for a total of $O(N^2)$ floating point operations per second per time step. If the maximum number of timesteps is independent of system size, the total time complexity is then $O(N^2)$. Memory complexity also scales as $O(N^2)$, since the algorithm requires the storage of 4 length $N^2$ floating point elements, and a $2N \times N^2$ sparse matrix with $N^2$ non-zero elements. 

\subsection{Performance Comparison Between CD and Memcomputing}

Here, we want to show that the RBM energy sampled by the memcomputing approach is consistently better than the one found by CD independently of the size of the RBM. Even though memcomputers are ideally realized as physical devices in hardware~\cite{13_memcomputing,DMM2}, here we compare their performance as {\it numerically integrated} dynamical systems versus traditional algorithmic methods (e.g., CD). 

We then first compute an ``exchange rate'' between one iteration of the numerical integration and $k$ steps of CD, such that the resulting computational complexity (i.e., wall time on the same processor) will be essentially identical. We discover empirically, across a large range of system sizes, that this exchange rate is about 30 steps of CD per iteration of the dynamics described by Eqs.~\ref{eq:voltages},~\ref{eq:xfast}, and~\ref{eq:xslow}. 

We choose as our test problems a set of randomly initialized $N\times N$ RBMs, with all weights sampled from a normal distribution with $\mu = 0$ and $\sigma^2 = 0.01$. The system sizes ranged from $N=100$ to 1000, which we chose to be large enough to observe the scaling in time. We then compute the relative energy differences, in percentage, $\Delta \epsilon \% = 100*(E^{Mem} - E^{CD})/E^{CD}$, between the energy $E^{Mem}$ obtained with the memcomputing ODEs described above, as compared to the energy $E^{CD}$ obtained using CD-$k$. For a direct comparison, we have run the memcomputing solver for $N_{tot} = 2N$ integration steps, scaled with the system size. Contrastive divergence was then run using the empirical exchange rate, $k = 30*N_{tot}$, resulting in the same computational cost seen in Fig.~\ref{fig:cd_vd_dynamics} (right panel). 

\begin{figure}[t!]
	\centering
	\includegraphics[width=0.95\textwidth]{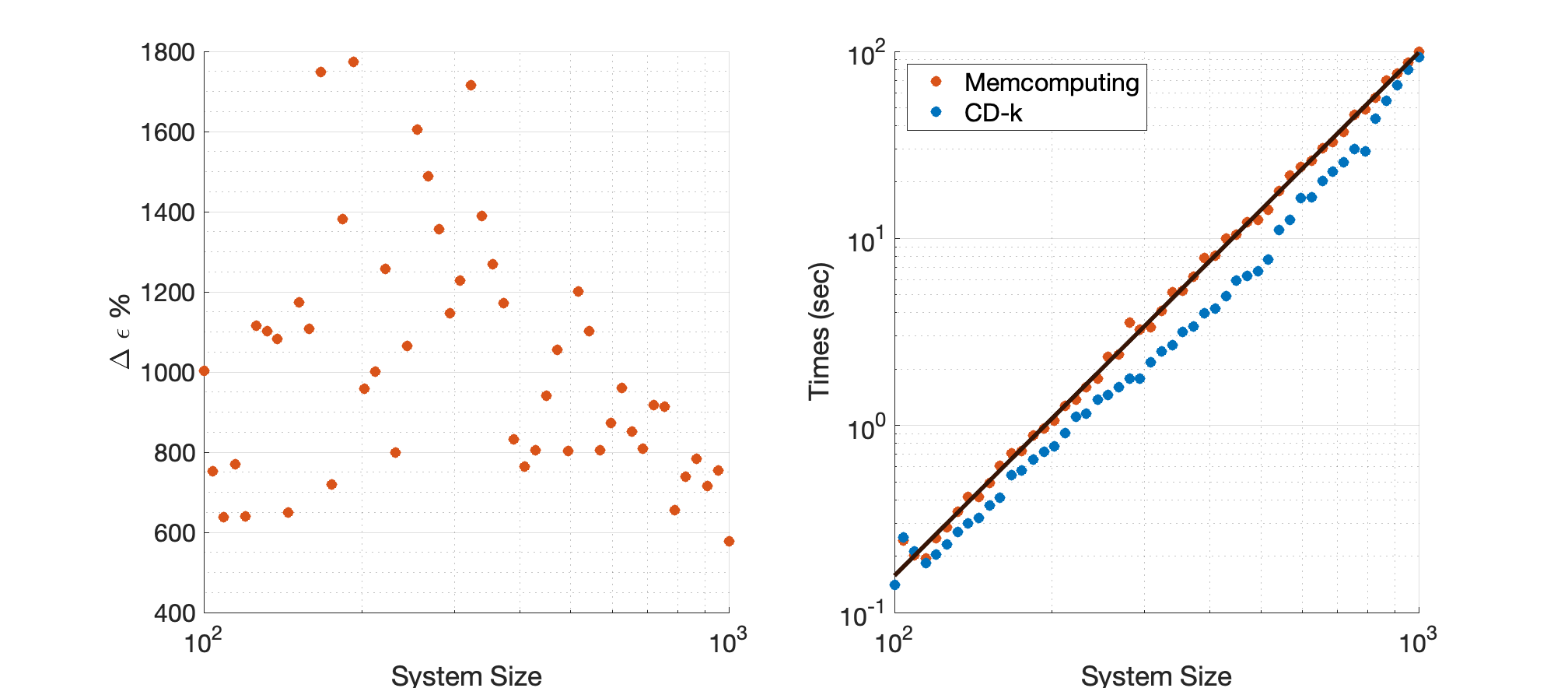}
	\caption{The median relative energy differences, in percentage, $\Delta \epsilon \% = 100*(E^{Mem} - E^{CD})/E^{CD}$ (left panel) between the memcomputing solver and CD-$k$, and respective wall clock times (right panel) from 20 randomly initialized $N\times N$ RBMs, with system size, $N$, ranging from 100 to 1000. A best fit line (slope $\sim 2.8$) of the memcomputing wall clock times is also plotted (dark line). Both calculations have been done on a single core of an AMD EPYC 7401 server.}
	\label{fig:cd_vd_dynamics}
\end{figure}

The energy results are plotted in Fig.~\ref{fig:cd_vd_dynamics} (left panel), where the memcomputing dynamics perform very favorably in terms of energies obtained compared CD-$k$, 
consistently above 400\%, often showing an improvement of more than 1000\%. In terms of time complexity (right plot), both algorithms follow the same linear trend on a log-log plot, indicating a polynomial scaling. Indeed, the best fit asymptotic behavior of both algorithms is almost cubic. This is consistent with our complexity analysis in the last section. Since both algorithms have a leading order scaling of $O(N^2)$ for a {\it fixed number of iterations}, they would scale cubically if we allowed the number of iterations to grow as $N$, the system size.
Finally, we want to stress that the set of equations used in the present work are only an example of how to implement a memcomputing solver and have not been optimized in terms of both speed and performance. 

\section{Stability of Mode-Assisted Training}

The stability of a pre-training procedure to training neural networks is a very desirable feature. This is because the KL divergence cannot be monitored during the pre-training process for a realistically sized RBM, so it is crucial for us to ensure that the KL divergence does not diverge. In this section, we show that using the mode in the model update term will guarantee convergence to a uniform distribution, and there is an optimal learning rate that provides the largest rate of convergence, with the learning rate being easily computable. \\

Note that in this work, the data term of a mode-assisted update is the same as traditional CD algorithms, so the difference is entirely in the way that the model term is approximated. Therefore, we only have to focus on the model term (which is ``approximated" by the mode of the joint distribution), and point out some of its key properties, in particular those pertaining to the stability of the pre-training procedure.

\subsection{Gauging the RBM}
\label{gauge}

For the sake of simplicity, we consider an $n\times m$ unbiased RBM with nodal values of $\mathbf{v}\in \{-1,1\}^n$ and $\mathbf{h}\in \{-1,1\}^m$, then the RBM energy is given by:
\begin{equation*}
E(\mbf{v}, \mbf{h}) = -\sum_{ij}W_{ij}v_ih_j.
\end{equation*}
Note that an RBM with nodal values $\{0,1\}$ can be trivially transformed into one with nodal values $\{-1,1\}$. For the analysis in this appendix, we will always assume (unless specifically mentioned) that an RBM is unbiased and equipped with nodal values $\{-1,1\}$. \\

Since in our work, we are interested in particular to the mode of the joint distribution, which is equivalently the nodal configuration that minimizes the RBM energy, we give a special denotation to this configuration, $\{ \mbf{v}^*, \mbf{h}^* \}$, and name it the {\it ground state} of the RBM energy. 

\begin{definition}[Ground State Energy]
Given an $n\times m$ RBM with weights $\mbf{W}$, we denote the {\bf ground state} of this RBM to be
\begin{equation*}
\{ \mbf{v}^*, \mbf{h}^* \} = \argmin_{\{\mbf{v}, \mbf{h}\}} \big[ E(\mbf{v},\mbf{h}) \big].
\end{equation*}
Furthermore, we denote the {\bf ground state energy} to be
\begin{equation*}
E_0(\mbf{W}) = -\sum_{ij} W_{ij}v_i^*h_j^*.
\end{equation*}
\end{definition}

Note that in practice, the ground state of an RBM can be thought of as being unique. In fact, for randomly initialized weights, the probability of having two or more minimal energy states, or {\it degenerate ground states}, is of measure zero. In theory, if there were to be multiple ground states, we can randomly select one of them to be $\{\mbf{v}^*, \mbf{h}^*\}$, and our analysis will not be affected at all. \\

Note that for any RBM, we can always map it to an equivalent RBM such that the ground state is $\mbf{+1}$. This is called a {\it gauge} operation, which we formally define as follows

\begin{definition}[Gauged RBM]
Given an $n\times m$ RBM with weights $\mbf{W}$ and ground state $\{ \mbf{v}^*, \mbf{h}^* \}$, we define the {\bf gauge} mapping $G: \mbb{R}^{nm} \mapsto \mbb{R}^{nm}$ such that $\mbf{W'} = G(\mbf{W})$ satisfies the following condition:
\begin{equation*}
W'_{ij} = W_{ij}v_i^*h_j^*.
\end{equation*}
Then we call $\mbf{W'}$ a {\bf gauged} RBM.
\end{definition}

\begin{remark}
Note that by this definition, it is easy to see that the ground state of any gauged RBM must be $\mbf{+1}$. This means that the ground state energy of a gauged RBM is simply the sum of its weights
\begin{equation*}
E = -\sum_{ij} W'_{ij}.
\end{equation*}
Furthermore, note that the form of the weight update equation is invariant under conjugation. In other words, if we let $f: \mbb{R}^{nm} \mapsto \mbb{R}^{nm}$ denote one iteration of weight update, then it is clear that
\begin{equation*}
f = G^{-1} \circ f \circ G.
\end{equation*}
This means that the dynamics of $\mbf{W}$ can be analyzed in terms of the dynamics of $\mbf{W}'$. In this section, we will always assume that the RBM is gauged. \\
\end{remark}

For a gauged RBM, the change of the weight elements (under unit learning rate) as a result of an iteration of mode-informed update is:
\begin{equation}
\label{w_up}
\delta W_{ij} = -\braket{v_i h_j}_{mode} = -v_i^* h_j^* = -1.
\end{equation}
Therefore, we see that every weight element is decremented by 1 uniformly across the entire weight matrix, and the energy change of the ground state energy is:
\begin{equation}
\label{dE0}
\delta E_0 = -\sum_{ij}\delta W_{ij} = nm.
\end{equation}

\subsection{Metric}
\label{met}

In order to investigate how the {\it joint probability mass function} (joint PMF), or $\pvh$, evolves under mode training, we have to look at how the energy changes for all nodal states. To do so, it is useful to define a distance measure between two states, to have a sense of how ``far apart" the two states are. We then propose the following distance measure.

\begin{definition}[Metric]
\label{distance}
We define a {\bf spin state} to be an ordered $(n+m)$-tuple given by $\mbf{s} = \{\mbf{v}, \mbf{h}\}$. Given two spin states, $\mathbf{s_1} = \{\mathbf{v_1},\mathbf{h_1}\}$ and $\mathbf{s_2} = \{\mathbf{v_2},\mathbf{h_2}\}$, we let $n_v = \frac{|\mathbf{v_2}-\mathbf{v_1}|^2}{2}$ and $m_h=\frac{|\mathbf{h_2}-\mathbf{h_1}|^2}{2}$, we define the distance to be
\begin{equation}
\label{dis}
d(\mbf{s_1},\mbf{s_2})=\frac{n_v}{n}+\frac{m_h}{m}-2\frac{n_vm_h}{nm}.
\end{equation}
\end{definition}

\begin{remark}
Note that $n_v$ simply counts the number of visible nodes that are different between the two states, and $m_h$ counts the number of different hidden nodes that are different. Note that the space of $\mbf{s}$ with this distance definition is a pseudometric space, in the sense that it is possible for the distance between two distinct points to be zero, in particular states that are related by $\mbb{Z}_2$ symmetry (or global spin flips). This can be easily verified by letting $\mbf{s_2} = -\mbf{s_1}$, giving us $n_v = n$ and $n_m = m$, and $d(\mbf{s_1},\mbf{s_2})=0$. In this pseudometric space, the distance $d$ is a measure of how ``different" two spin states are up to a $\mbb{Z}_2$ symmetry. A formal discussion of this metric, including a proof of triangle inequality, is provided in Appendix C of our related work \cite{rbm_loops}. The usefulness of defining the metric this way will be apparent in proposition \ref{t1}. \\
\end{remark}

\begin{remark}
It is important to note that $\{n_v, m_h\}$ is not uniquely determined by $d$. To see this clearly, we rewrite Eq.~(\ref{dis}) in terms of the following Diophantine equation
\begin{equation*}
(2n_v - n)(m-2m_h) = (2d-1)nm,
\end{equation*}
solving for integers $n_v \leq n$ and $m_h \leq m$. It is easy to see that this equation is over-determined by realizing that it is possible for the RHS to have multiple prime factors.
\end{remark}

\subsection{Energy Change}
\label{eng_change}

Equation (\ref{dE0}) gives the change in the energy of the ground state under a mode-assisted update iteration. However, to analyze the stability of the training procedure, it is necessary to look at the energy change of all states. To simplify our discussion, instead of looking at the energy of each individual state, let us consider the average energy of all the states distance $d$ from the modal configuration, which we denote as $\overline{E}(d)$. Note that the average is not the expected value over the joint PMF $\pvh$. Rather, it is an unweighted average (or the expected value over a uniform probability measure). It is interesting to note that this average energy only depends on the ground state energy and the distance $d$ from the ground state.

\begin{proposition}[Average Energy]
\label{t1}
The average energy of states distance $d$ from the ground state is:
\begin{equation*}
\overline{E}(d)=(1-2d)E_0.
\end{equation*}
\end{proposition}

\begin{proof}
Given some distance $d$, there can be multiple assignments of $\{n_v,m_h\}$ that correspond to this distance. However, if given a particular tuple $\{n'_v,m'_h\}$, we show that the average energy of all states with spins differing from the ground state by $\{n'_v,m'_h\}$ is only dependent on the distance $d'$ corresponding to the tuple, then the average energy of states of distance $d'$ from the ground state is simply the average energy of states of with spins differing from the mode by $\{n'_v,m'_h\}$. \\

The average energy of states with spins differing from the ground state by $\{ n'_v, m'_h \}$ can be expressed as
\begin{equation*}
\E_{\{n'_v,m'_j\}}(E) = \E_{\{n'_v,m'_j\}} \big( \sum_{ij}W_{ij}v_ih_j \big) 
= \big( \sum_{ij} W_{ij} \big) \E_{\{n'_v,m'_j\}}(v_1h_1),
\end{equation*}
where in the last equality, we used the linearity of the expected value and the symmetry of the RBM. We easily see that the marginal probability distribution of a single spin is given by (with the underlying joint distribution being uniform) 
\begin{equation*}
\begin{split}
&\PP(v_1 = +1) = \frac{n-n'_v}{n}, \quad \PP(v_1 = -1) = \frac{n'_v}{n}, \\
&\PP(h_1 = +1) = \frac{m-m'_h}{m}, \quad \PP(h_1 = -1) = \frac{m'_h}{m},
\end{split}
\end{equation*}
which gives us
\begin{equation*}
\E_{\{n'_v,m'_j\}}(E)
= \big[(1-2\frac{n'_v}{n})(1-2\frac{m'_h}{m})\big] \big[ -\sum_{ij}W_{ij} \big]
= E_0(1-2d').
\end{equation*}
Therefore, the average energy of states distance $d'$ from the mode is also $\overline{E}(d') = E_0(1-2d')$.
\end{proof}

Since the average energy distance $d$ from the ground state is only dependent on $d$, we expect this to be true also for the change in energy for a state at a distance $d$ from the ground state, under the weight update routine given in Eq.~(\ref{w_up}).

\begin{proposition}[Energy Change]
\label{t2}
Given any state distance $d$ from the ground state, the change in the energy of that state is given by
\begin{equation*}
\delta E(d) = nm(1-2d).
\end{equation*}
\end{proposition}

\begin{proof}
Again, we only have to focus on one particular assignment of the tuple $\{n_v,m_h\}$ which corresponds to the distance $d$, and show that the change in energy of a state corresponding to that tuple depends only on $d$. Without loss of generality (WLOG), we assume that the first $n_v$ visible nodes are of value $-1$, and the first $m_h$ hidden nodes are of value $-1$. Then the change in energy is given by:
\begin{equation*}
\begin{split}
\delta E(d) &= -\sum_{ij} \delta W_{ij}v_ih_j \\
&= \sum_{ij} v_ih_j \\
&= \sum_{i=1}^{n_v}\sum_{j=1}^{m_h} v_ih_j
+ \sum_{i=1}^{n_v}\sum_{j=m_h+1}^{m} v_ih_j
+ \sum_{i=n_v+1}^{n}\sum_{j=1}^{m_h} v_ih_j
+ \sum_{i=n_v+1}^{n}\sum_{j=m_h+1}^{m} v_ih_j \\
&= n_vm_h - n_v(m-m_h) - (n-n_v)m_h + (n-n_v)(m-m_h)\\
&= 4n_vm_h - 2n_vm - 2nm_h + nm  \\
&= nm(1-2d),
\end{split}
\end{equation*}
where we have used the fact that $\delta W_{ij} = -1$ from Eq.~(\ref{w_up}). 
\end{proof}

\begin{remark}
Note that the energy change is only dependent on the size of the RBM and the distance $d$ from the ground state, so all the states at distance $d$ experience the same energy change. Under a given learning rate $\gamma$, the actual energy change is then
\begin{equation*}
\delta E(d) = \gamma nm(1-2d).
\end{equation*} 
Combining propositions \ref{t1} and \ref{t2}, we see that the energy change can be alternatively written as
\begin{equation}
\label{nm_avg}
\delta E(d) = \gamma nm \frac{\overline{E}(d)}{E_0}.
\end{equation}
\end{remark} 

At this point, it is necessary to take an intermission to look at the role that the mode update term plays in the pre-training procedure. From Eq.~(\ref{nm_avg}), we see that the energy change of a state distance $d$ from the ground state is proportional to the average energy of the states at the same distance $\overline{E}(d)$. In the context of the entire pre-training procedure, this energy change can be interpreted as a constant drift term that pulls the energy back to zero with strength proportional to the average energy of all the states of the same distance. Loosely speaking, the joint distribution will become more uniform under an iteration of mode-assisted update. \\

Note that this behavior can also be achieved with standard regularization procedures such as an exponential weight decay term like $\delta W_{ij} = -W_{ij}$. However, such regularization techniques are usually undesirable as they do not induce an effective sampling of a multi-modal distribution. Our procedure, however, does not suffer from such drawbacks, and in fact promotes the effective sampling of a multi-modal distribution  (see section \ref{eff}).

\subsection{Approaching Uniformity}

In this section, we formalize the argument that the RBM energies over all states become more uniform under a mode-assisted update iteration. To do so, we mainly focus on the energy variance across all states, and show that it must decrease under a suitable learning rate. This statement can be made more precise as follows.

\begin{theorem}[Decrease in Energy Variance]
\label{uni_var}
If $0 < \gamma < -\frac{2E_0}{nm}$, then the variance of the energies $\mathrm{Var}_{\mathbf{s}}(E(\mathbf{s}))$ over all spin states decreases. The largest decrease in variance occurs when $\gamma = -\frac{E_0}{nm}$.
\end{theorem}

\begin{proof}
We reiterate the fact that the underlying PMF for the states is assumed to be uniform, or $f(\mathbf{s})=\frac{1}{2^{n+m}}$ for every nodal configuration $\mathbf{s}$. We can then define a random variable $D$ with its PMF being:
\begin{equation*}
f_D(d) = \frac{1}{2^{n+m}}\sum_{d(n_v,m_h)=d}{n \choose n_v}{m \choose m_h},
\end{equation*}
which can be interpreted as the probability of a randomly chosen state to be a distance $d$ from the ground state. From this PMF expression, we can easily derive the expected value and the variance of the distance of two randomly chosen states
\begin{equation}
\label{expvar}
\E(D) = \frac{1}{4}, \qquad \Var(D) = \frac{1}{4nm},
\end{equation}
where we see that the variance is small relative to the expectation value for a large system. We then use the law of total variance to write the variance of the energies over all states as
\begin{equation}
\label{var_2}
\begin{split}
\Var(E(\mbf{s})) 
=\, & \E_{D}\big[ \Var_{\mbf{s}}(E(\mbf{s}) \cond d(\mbf{s}) = D) \big] \\
+\, & \Var_D\big[ \E_{\mbf{s}}( E(\mbf{s}) \cond d(\mbf{s}) = D) \big].
\end{split}
\end{equation}
We first begin by focusing on the first term. Note that the term $\Var_{\mbf{s}}(E(\mbf{s}) \cond d(\mbf{s}) = D)$ is the conditional variance of energies of the states distance $D$ from the mode. If we update the energies according to Eq.~(\ref{nm_avg}), then the new variance can be written as $\Var_{\mbf{s}}(E(\mbf{s})+\gamma nm(1-2D) \cond d(\mbf{s}) = D)$. The term $\gamma nm (1-2D)$ is dependent only on $D$ but not the specific nodal configuration $\mathbf{s}$, so it is just a constant offset in the context of the conditional variance, and the variance will remain constant. Therefore, the first term of the variance decomposition is constant, and we only have to focus on the second term, which can be conveniently written as: 
\begin{equation*}
\begin{split}
\Var_D(\overline{E}(D)) 
&= \Var_D(E_0(1-2D)) \\
&= 4E_0^2\Var(D) = \frac{E_0^2}{nm}.
\end{split}
\end{equation*}
After a weight update, this variance becomes
\begin{equation}
\label{quadE}
\begin{split}
\Var_{D}(\overline{E}(D)+\gamma nm \frac{\overline{E}(d)}{E_0}) 
=& 4E_0^2(1+\frac{\gamma nm}{E_0})^2\Var(D) \\
=& \frac{E_0^2}{nm}(1+\frac{\gamma nm}{E_0})^2.
\end{split}
\end{equation}
In this form, it is easy to see that the variance decreases when the learning rate satisfies
\begin{equation}
\label{gammabound}
0<\gamma<-\frac{2E_0}{nm},
\end{equation}
with the largest decrease being $\delta\Var_D(\overline{E}(D))=4E_0^2\Var_D(D)=\frac{E_0^2}{nm}$, which occurs at the learning rate $\gamma = -E_0/nm$. This is then our {\it optimal} learning rate.
\end{proof}

\begin{remark}
To avoid confusion, note that $E_0$ is negative, so $-\frac{2E_0}{nm}$ is positive, so the learning rate $\gamma$ is bounded in some positive interval. Note that the two biases for the visible and hidden spins can be expressed as two {\it ghost spins}~\cite{rbm_loops}, thereby effectively adding one more spin to each layer. By taking into account the biases, we see that the largest decrease of the variance occurs when 
\begin{equation}
\label{gammaapp}
\gamma \approx -E_0/(n+1)(m+1),
\end{equation}
which is what we use in the main text. \\
\end{remark}

There are two important things to note here. First, the learning rate, as presented in Eq.~(\ref{gammaapp}) is generally very large and is only {\it optimal} in the sense that it provides the fastest convergence to a uniform joint PMF, which is desirable for a {\it stable} pre-training routine, but not necessarily optimal for minimizing the KL divergence. The practical usefulness of Eq.~(\ref{gammaapp}) is to mainly provide an upper bound to the learning rate that ensures stability. It should be noted that the analysis ignores the presence of the data term (see Eq. (6) in the main text) and is only carried out over a single iteration; in other words, it may be possible that a large learning rate will force the system into a local minimum in the KL divergence rather quickly. Therefore, in the practical setting a smaller learning rate would be more beneficial. In the main paper, we then {\it normalized} this learning rate with the learning rate of CD, which results in $\epsilon_{CD}\gamma<\gamma$ (as $\epsilon_{CD}<1$). 

The second thing to note is that Eq.~(\ref{gammaapp}) is not exact as the {\it ghost spins} are fixed nodes that cannot be ``flipped", so theorem \ref{t1} no longer applies, meaning that the average energy of states distance $d$ from the ground state can no longer be uniquely determined by $E_0$ and $d$ alone. Nonetheless, for large RBMs, the contribution from biases are relatively small, and the approximation is close to exact. 

\subsection{Suboptimal Updates}

Before we conclude this section, we make two final remarks concerning suboptimal updates, or updates that are not informed by the global mode directly. The first remark pertains to a practical setting where locating the global mode is difficult or too computationally expensive, and only an {\it approximate} mode can be obtained, or a state with energy close to the ground state. We discuss how an update informed by this state still ensures stability. The second remark compares a mode-assisted update with an update with the model term sampled by some form of stochastic algorithm (such as CD), and we show that the latter update procedure does not ensure stability. \\

Note that in Eq.~(\ref{gauge}), we transformed the weight elements such that the ground state is $\mathbf{v}^* = \mathbf{+1}$ and $\mathbf{h}^* = \mathbf{+1}$. However, this procedure is general and can be done for any given state. Given any two states, $\mathbf{v}_1$ and $\mathbf{h}_1$ with some associated energy $E_1$, it is always possible to gauge the RBM in a way such that $\mbf{v_1}=\mbf{+1}$ and $\mbf{h_1}=\mbf{+1}$. The previous proofs will still carry through for $E_1$ as long as $E_1<0$. This means that the mode training procedure {\it does not} hinge on the fact that the weight update has to be informed by the exact ground state, and any state sufficiently close to the ground state should suffice. However, it should be noted that using the ground state to inform the weight update provides the greatest decrease in energy variance since the maximum of $\delta\Var_D(\overline{E}(D))$ scales quadratically with $E_0$ (see Eq.~(\ref{quadE})). \\

Note that in theorem \ref{uni_var}, the argument that the conditional variance of the energies conditioned on some distance $d$ from the ground state does not change is based on the fact that the weights are updated uniformly across the RBM according to Eq.~(\ref{w_up}). However, for a stochastic algorithm, the weight updates are clearly not uniform (or even deterministic for that matter), so nothing can be said about the change of the conditional variance. It is possible for the conditional variance to increase under a stochastic update, thus pulling the energies away from uniformity if the magnitude of the increase overcomes the decrease in the second term in Eq.~(\ref{var_2}) (the ground state variance). \\

To conclude this subsection, we discuss briefly the contribution of the data term in updating the weight matrix. Clearly, if we look at the gauged RBM matrix, the change in each element generated by the data term is bounded above by $+1$, meaning that its contribution cannot overcome the guaranteed $-1$ decrease generated by the mode update term. This means that it is impossible for the ground state energy to decrease even in the presence of the data term, so the mode of the joint distribution must not increase, thus the training never diverges. This effectively ensures the global stability of our mode-assisted training method. 


\section{Efficient Sampling of Multi-modal Distributions}
\label{eff}

So far, we have shown that our update procedure guarantees stability. However, as briefly mentioned at the end of section \ref{eng_change}, stability is also guaranteed by standard regularization terms such as the weight decay term, $\delta W_{ij} = -W_{ij}$. In this section, we make the crucial distinction between our procedure and standard weight regularization procedures by pointing out the key phenomenon that our procedure is capable of efficiently exploring the landscape of a multi-modal PMF. \\

This property of the mode-training method is most readily analyzed from the perspective of the {\it frustration index} of the RBM instance. The frustration index can be interpreted as a measure of the difficulty of discovering the nodal ground state of a given RBM instance, and interestingly, an increase in the frustration index is correlated with an increased rate of exploration of the multi-modal distribution. Therefore, in some sense, for a given iteration of weight updates, the difficulty of finding the mode of that distribution is ``compensated" by an increased efficiency of PMF exploration. \\

We begin by formally defining the frustration index, followed by a brief explanation of how the mode-training algorithm explores efficiently the PMF. Finally, we relate the two concepts in a cohesive manner. We provide an extensive analysis on the frustration of the RBM and its practical applications in our related work \cite{rbm_loops}.

\subsection{Frustration Index}
\label{frus}

The {\it frustration index} is the ratio between the sum of unsatisfied couplings at the ground state and the sum of all coupling strengths. Formally, for a gauged RBM, it can be defined as follows
\begin{equation*}
f = \frac{1}{2}\left[ \frac{\sum_{ij} |W_{ij}| - \sum_{ij} W_{ij}}{\sum_{ij} |W_{ij}|} \right].
\end{equation*}
This index is closely related with the degeneracy of the low-energy states. In other words, with an increase in the frustration index, the excited states will be spaced closer to the ground state in energy. Furthermore, for a highly frustrated system, the transition from the ground state to the excited states usually involves flipping a large cluster of nodes. This gives rise to a large population of local minima in the energy landscape spaced {\it far apart in distance but close together in energy} (in terms of the metric discussed in section \ref{met}), and this property of a highly frustrated system makes it difficult for local search algorithms to locate the global minimum. This motivates the need for an algorithm that is able to learn the long-range correlations of the RBM spins, and a possible candidate of this algorithm is presented in section \ref{mem}.

\subsection{Inefficiency of Weight Decay}

In this section, we discuss briefly why the standard weight decay algorithm $\delta W_{ij} = -\gamma W_{ij}$ (where $\gamma$ is some learning rate) is not efficient in assisting local algorithms in sampling a multi-modal distribution. To begin with, we first recall that the joint distribution of the RBM is
\begin{equation*}
\pvh = \exp( -E(\mbf{v}, \mbf{h}) ),
\end{equation*}
where $E(\mbf{v}, \mbf{h}) = \sum_{ij} W_{ij}$ for a gauged RBM. Note that the weight decay update is a contracting affine transformation of the energies of all states, or simply a rescaling of the energies by some constant $\beta = (1 - \gamma) < 1$, meaning that the joint distribution transforms as
\begin{equation*}
\pvh \rightarrow \pvh^{\beta},
\end{equation*}
where the normalization condition is ignored. \\

Of course, the distribution does become more uniform under this transformation; however, the {\it ordering} of the states with respect to their energies will not change, meaning that the ordering of the dominant modes remains invariant under this transformation. In other words, a poorly initialized Markov chain trapped under a dominant mode will still remain trapped unless $\beta$ becomes sufficiently small; this means that a large learning rate, $\gamma$, is required to free the Markov chain and allow efficient exploration of the joint distribution. However, a large learning rate in this context is undesirable, as it brings the RBM to uniformity in a drastic manner, which voids much of the information gained from the previous iterations of pre-training.
The inefficiency of this approach boils down to the indiscriminate update of the weight matrix that is ignorant of the energy ordering of the states or the distance between them (see definition \ref{distance}). \\

Our mode-assisted update, on the other hand, updates the weight matrix based on the ground state configuration of the RBM, resulting in a maximal increase in energy for the ground state, and the energy change is ``propagated" to the other states based on their distances from the ground state (see proposition \ref{t2}). An entirely different energy landscape will then emerge under this update procedure even under a small learning rate, and it is likely that a new ground state at a faraway distance will ``pop" up. The next update iteration is then based on this new found mode, and the process is repeated. Effectively, we are {\it dynamically} sampling the energy landscape by making large leaps between dominant states without resorting to forcing uniformity on the energies.

\subsection{Global Mode Cycling}

For the sake of simplicity, consider a gauged RBM with a joint distribution having three dominant states, with their RBM energies being $E_0 < E_1 < E_2$. The heuristic analysis in this section can be easily generalized to multi-modal distributions with arbitrary number of dominant states. We can also assume that the pairwise distances between the three modes are the same (meaning that the three modes form an equilateral triangle under the metric defined in definition \ref{distance}), which we can then denote simply as $d$. \\

If we assume that the learning rate is $\gamma$, then from Eq.~(\ref{nm_avg}), we see that the new energies of the three states will become
\begin{equation*}
\begin{split}
E_0^{(1)}&=E_0^{(0)} + nm\gamma \\
E_1^{(1)}&=E_1^{(0)} + nm\gamma(1-2d) \\
E_2^{(1)}&=E_2^{(0)} + nm\gamma(1-2d),
\end{split}
\end{equation*}
where we are assuming that the magnitude of the learning rate is much larger than the energy gaps \footnote{This is a justified assumption if the system is highly frustrated, as the energy gaps near the ground state are generally very small for such system.}, or more precisely
\begin{equation}
\label{gamma}
\gamma > \frac{E_1^{(0)} - E_0^{(0)}}{2nmd},
\end{equation}
where we note that the lower bound of gamma is proportional to the energy difference between the first excited state and the ground state. This guarantees that after one update, the ordering of the new energies of the states will become
\begin{equation*}
E_1^{(1)} < E_2^{(1)} < E_0^{(1)},
\end{equation*}
which means that $E_1^{(1)}$ is the new ground state energy, and the next iteration of weight update will be based on state $E_1^{(1)}$, resulting in the following new energies
\begin{equation*}
\begin{split}
E_1^{(2)}&=E_1^{(1)} + nm\gamma \\ 
E_2^{(2)}&=E_2^{(1)} + nm\gamma(1-2d) \\
E_0^{(2)}&=E_0^{(1)} + nm\gamma(1-2d),
\end{split}
\end{equation*}
The energies are then reordered as
\begin{equation*}
E_2^{(2)} < E_0^{(2)} < E_1^{(2)},
\end{equation*}
so $E_2^{(2)}$ becomes the new ground state energy. And similarly, the third iteration will recover the original energy ordering $E_0^{(3)} < E_1^{(3)} < E_2^{(3)}$. \\

Therefore, we see that in general, whenever we perform a weight update, the energy ordering of the modal states will experience a left circular shift, so we are, in some sense, sampling the multiple modes in a {\it cyclic fashion}, which allows us to effectively cover a large volume of the probability measure. 

\subsection{Relationship between Frustration and Mode Sampling}

Now, we discuss how an increase in the frustration index is conducive to an efficient sampling of the multi-modal distribution. We here consider simply a gauged $n\times n$ RBM, with ground state energy $E_0$. We denote the average energy of states distance $d$ from the ground state as $\overline{E}(d)$ (see proposition \ref{t1}). Under an iteration of mode update, the new energies are (see Eq.~(\ref{nm_avg}))
\begin{equation*}
E_0' = E_0 + n^2\gamma 
\qquad
\overline{E}(d)' = \overline{E}(d) + n^2\gamma(1-2d).
\end{equation*}

\subsubsection{Small Frustration}

For the sake of simplicity, consider the case where the frustration index of the RBM is zero, then all the weights can be assumed positive. Furthermore, we make the simplifying assumption that the weights are iid\footnote{From here on, {\it iid} will serve as the abbreviation for {\it independent and identically distributed}.} random variables with uniform distribution in $[0,1]$. We then make the following claim.

\begin{proposition}
If we update the weight matrix continuously with the mode-assisted update procedure, then the new ground state will differ from the original ground state by distance $d \sim \frac{1}{n}$ almost surely. \\

More formally put, if we let $E_{\min}(d)$ be the minimum energy of states distance $d$ from the ground state, and $\Delta E(d) = E_{\min}(d) - E_0$. Then the smallest learning rate for which a new ground state can emerge is
\begin{equation*}
\gamma' = \inf \{ \gamma \cond \exists d\in (0,1], \quad 2d n^2\gamma  = \Delta E(d) \}.
\end{equation*}
If we let $d'$ be the distance such that $2d'n^2\gamma = \Delta E(d')$, then
\begin{equation*}
\lim_{n \to \infty} \PP( d' > \frac{1}{n} ) = 0.
\end{equation*}
\end{proposition}

\begin{proof}

Given any state distance $d$ from the ground state, we have
\begin{equation*}
\PP\big( 2dn^2\gamma > E(d) - E_0 \big) = \frac{1}{2}\erfc\big( \sqrt{6}(1-2\gamma)nd \big).
\end{equation*}
WLOG, we can also assume that $n$ is a prime number, then the number of states distance $\frac{k}{n}$ from the ground state (where $k<n$) is $2{n \choose k}$, which gives us (denoting $\beta = \sqrt{6}(1-2\gamma)$ and $k = nd$).
\begin{equation*}
\begin{split}
& \PP\big( 2dn^2\gamma > \Delta E_{\min}(d) \big) \\
= & 1 - \Big[ 1 - \frac{1}{2}\erfc\big( \sqrt{6}(1-2\gamma)nd \big) \Big]^{2 {n \choose k} } \\
\sim & 1 - \exp\Big[ -{n \choose k} \frac{e^{-\beta k^2}}{\sqrt{\pi} \beta k} \Big] \\
\equiv & J(n,k,\beta).
\end{split}
\end{equation*}
Note that $\forall \epsilon \in (0,1)$, we let $\beta'$ such that $J(n,1,\beta') = 1-\epsilon$, then we have
\begin{equation*}
\forall k \in [2,n], \qquad
\lim_{n \to \infty} J(n,k,\beta') = 0,
\end{equation*}
which proves the proposition.
\end{proof}

This result implies that in the limit of large $n$, the new ground state is only likely going to differ from the old ground state by distance $d\sim \frac{1}{n}$, so we are only moving away from the old ground state by a very small distance. This means that a small frustration is not conducive to an efficient sampling of the phase space.

\subsubsection{Large Frustration}

A highly frustrated system is generally hard to study, so we here provide a brief heuristic argument for the efficient sampling of the PMF for a highly frustrated RBM. Recall that in the case of large frustration, the first excited state differs from the ground state by a large number of nodes (hence a large distance $d$) but by only a small amount of energy. Also recall from Eq.~(\ref{gamma}) that the lower bound of the learning rate scales proportionally to the energy difference and inversely proportionally to the distance. Putting the two results together, we see that in order for the first excited state to become the new ground state, we only require a very small learning rate (which is conducive to a faster convergence of the KL-divergence), and furthermore, transitioning from the ground state to the new ground state effectively allows us to traverse a large distance, which allows us to efficiently sample the full PMF.


\section{Defining Modal Correspondence}

The goal of this section and the following is to show that the mode of the marginal distribution of the visible layer, $\pv$, and the mode of the joint distribution, $\pvh$, are {\it strongly correlated}. We will dedicate this section to a formal definition of this notion of correspondence, and provide a full proof of correspondence in the following section. For now, we can interpret this strong correspondence as the phenomenon that there is a high chance for the mode of $\pv$ and the mode of $\pvh$ to overlap in the configuration of $\mbf{v}$, meaning that the mode of $\pvh$ can be used to ``approximate" the mode of $\pv$. 

\subsection{Unnormalized PMFs}

Recall that we base the analysis in this section on an $n\times m$ unbiased RBM with nodal values of $\mathbf{v}\in \{-1,1\}^n$ and $\mathbf{h}\in \{-1,1\}^m$. The discussion in this section can be easily extended to a biased RBM. To ease the burden of notation, we first begin by defining an {\it angle} variable $\bs{\theta} = \mathbf{v}\cdot \mathbf{W}$, which allows us to rewrite the RBM energy and the joint probability mass function (PMF) as follows:
\begin{equation*}
\begin{split}
E &= -\bs{\ta}\cdot \mbf{h}, \\
\pvh &= \frac{1}{Z}e^{-E} = \frac{1}{Z}e^{\bs{\ta}\cdot \mbf{h}},
\end{split}
\end{equation*}
where $Z$ is the partition function of the RBM. The marginal PMF of the visible layer can be obtained by fixing the visible layer and summing the joint PMF over all the hidden layer configurations:
\begin{equation}
\label{sum}
p(\mbf{v}) = \sum_{\mbf{h}} p(\mbf{v},\mbf{h}) = \frac{1}{Z}\sum_{\mbf{h}} e^{\bs{\theta}\cdot \mbf{h}} = \frac{1}{Z}\prod_{j=1}^m 2\cosh(\theta_j),
\end{equation}
where the last equality is obtained by factoring the sum into each individual hidden nodes. \\

Since we are mainly concerned with the correspondence of the modal configurations instead of the normalized probability mass, we can simply ignore the constant prefactor $\frac{1}{Z}$ as the normalization prefactor and simply look at the unnormalized PMFs:
\begin{equation*}
\Pvh = e^{\bs{\theta}\cdot \mbf{h}} \qquad \Pv = \prod_{j=1}^{m}2\cosh{\theta_j},
\end{equation*}
where the use of the capital letter $P$ is to denote the unnormalized PMF. Note that since $p \mapsto P$ is an affine transformation, the ordering of the states in terms of their energies is invariant. \\

An issue we have to first address is that the nodal configuration of the joint distribution is described by the configurations of both layers $\{\mbf{v},\mbf{h}\}$, while the nodal configuration of the marginal distribution is only described by the visible layer $\mbf{v}$. So in order to compare the nodal configurations of the two PMFs, we have to relegate $\Pvh$ into a PMF that only depends on $\mbf{v}$, which we do as follows:

\begin{definition}
Given a PMF $\Pvh$, we denote
\begin{equation*}
\Qv = \max_{\mbf{h}}\Pvh,
\end{equation*}
\end{definition}

\begin{remark}
In other words, $\Qv$ is the maximum of the $\Pvh$ over all $\mbf{h}$ under some fixed $\mbf{v}$. 
Note that the purpose of this definition is to have the mode of $\Qv$ be the same as the mode of $\Pvh$ ``projected" onto the space of $\mbf{v}$. In other words, if we let $\{ \mbf{v}^*, \mbf{h}^* \}$ be the mode of the joint distribution $\Pvh$, then we have the following:
\begin{equation*}
\argv \Qv = \argv (\argh \Pvh) = \mbf{v}^*
\end{equation*}
This means that the mode of the joint distribution $\Pvh$ is the same as the mode of $\Qv$ in the $\mbf{v}$ component. \\
\end{remark}

\begin{remark}
Note that there is a bijection between the visible configurations and the angle variables given by $\bs{\theta} = \bs{v}\cdot \mbf{W}$, so we can make $Q$ depend on $\bs{\ta}$ instead, or $Q(\bs{\ta})$, which is usually the form that we will be using for this section. Similarly, we can also write $P(\bs{\ta})$ as the unnormalized marginal distribution.\\
\end{remark}

To simplify the analysis of modal correspondence, we first obtain a closed form expression for $\Qv$:

\begin{lemma}
\label{qv}
$\Qv = \exp(\sum_j |\theta_j|)$.
\end{lemma}
\begin{proof}
Note that the expression for $\Pvh$ can be written as $\Pvh = \exp(\sum_j \theta_j h_j)$. It then follows that $\argh \Pvh = \exp(\argh \sum_j \theta_j h_j) = \exp(\sum_{j=1}^m\argmax_{h_j} (\theta_j h_j))$. Since $h_j \in \{-1,1\}$, it is easy to see that $\argmax_{h_j}(\theta_j h_j) = |\theta_j|$. Therefore, we have $\Qv = \argh \Pvh = \exp(\sum_j |\theta_j|)$
\end{proof}

Now, if we denote $\vs$ as the $\mbf{v}$ component of the mode of $\Pvh$ and $\vd$ as the mode of $\Pv$, then the question of whether the marginal mode equals to the joint mode can be succinctly expressed as
\begin{equation*}
\vd \stackrel{?}{=} \vs.
\end{equation*}
The equality, in fact, does not hold in the absolute sense, and it is very easy to construct pathological examples to violate the equality. However, for practical purposes, we only need this equality to hold with some non-negligible probability for an RBM with weights randomly sampled from some distribution. We then formally define the notion of correspondence as follows

\begin{definition}
\label{defcor}
Given an $n\times m$ RBM with weights $\mbf{w}$ sampled from some distribution $f_{\mbf{W}}(\mbf{w})$, we say that the marginal mode and joint mode of the RBM are {\bf strongly correlated} if the following holds
\begin{equation}
\label{equiprob}
\PP\big[ \bigwedge\limits_{\mbf{v}\in \{-1,+1\}^n} P(\mbf{v}) \leq P(\vs)\big] \geq 0.5,
\end{equation}
where $\vs$ is the $\mbf{v}$ component of the mode of $\Pvh$. 
\end{definition}

\begin{remark}
First, we recall that $\vs$ is the $\mbf{v}$ component of the joint distribution $\Pvh$. If $\vs$ is also the mode of the marginal distribution $\Pv$, or $\vs = \vd$, then clearly we require that $\Pv \leq P(\vd) = P(\vs)$, for all $\mbf{v}$ configurations. In order to weaken the condition of exact modal correspondence, we simply require that the probability of the inequality, $\Pv \leq P(\vs)$, holds for all $\mbf{v}$ to be greater than some arbitrary value, which we chose to be $0.5$ here. 
\end{remark}

\subsection{Trivial Cases}

There are two cases where proving the modal correspondence is trivial; the two cases occur at the beginning and end of the pre-training respectively. At the beginning of the training, the frustration index is small for the RBM, and the system is trivially ferromagnetic. At end of the training, the magnitude of the weights are large, and the nodal activation of the hidden layer is almost certain. 

\subsubsection{Small Frustration}
\label{smallf}

If the frustration index is small, we can state the following.

\begin{proposition}
$\argv \Qv = \argv \Pv$ for an RBM with zero frustration.
\end{proposition}

\begin{proof}
We look at the gauged RBM where all weight elements are non-negative. Recall that the ground state of a gauged RBM is $\mbf{+1}$, then we have $\argvh \Pvh = \mbf{+1}$, which implies $\argv \Qv = \mbf{+1}$. Note that $\Pv = \prod_j 2\cosh(\ta_j) = \prod_j 2\cosh(\sum_i W_{ij}v_i) \leq \prod_j 2\cosh(\sum_i W_{ij}) = P(\mbf{+1})$, where the inequality comes from the fact that $W_{ij}\geq 0$ and $v_i\in \{-1,1\}$, so we have $\argv \Pv = \mbf{+1}$ as well. The proposition is then shown.
\end{proof}

\begin{remark}
Note that this proposition implies directly modal correspondence as defined in definition \ref{defcor} in the absolute sense.
\end{remark}

\subsubsection{Large Weights}
\label{largew}

Near the end of the RBM training, the magnitude of the weights are usually very large (thus also the magnitude of the elements of $\bs{\ta}$), and the activation of the hidden nodes becomes increasingly certain. Intuitively speaking, this means that given any visible configuration, there is only one dominant hidden configuration corresponding to it. Therefore, the marginal distribution $\pv$ (which involves the sum over all hidden configurations) can be effectively approximated with the joint distribution $\pvh$. We formalize this argument as follows:

\begin{proposition}
\label{largeta}
Given an $n\times m$ weight matrix, $\mbf{W}$, with the joint mode $\mbf{v}$ satisfying
\begin{equation*}
\forall j \in [[1,m]], \qquad
| \sum_i W_{ij}v_i | \neq 0,
\end{equation*}
and the ground state is not degenerate. Then $\exists M>0$, such that for an RBM with the weight matrix, $M \mbf{W}$, the following is true
\begin{equation*}
\argv \Qv = \argv \Pv.
\end{equation*}
\end{proposition}

\begin{proof}
We look at the gauged RBM so that the ground state is $\mbf{+1}$, then we set
\begin{equation*}
\theta_j = \sum_{ij}W_{ij} > 0.
\end{equation*}
Let $\mbf{v'}$ be the visible component of any other state, then we denote
\begin{equation*}
\theta'_j = \sum_{ij}W_{ij}v'_i.
\end{equation*}
Then the following must be true
\begin{equation*}
\exists \epsilon > 0, \qquad
\sum_j |\ta_j| - \sum_{j} |\ta'_j| = \epsilon.
\end{equation*}

Recall from proposition \ref{qv} that 
\begin{equation*}
Q(\bs{\ta}) = \prod_j \exp(|\ta_j|).
\end{equation*}
Furthermore, we can write the marginal distribution as
\begin{equation*}
P(\bs{\ta}) = \prod_j 2\cosh(|\ta_j|).
\end{equation*}
Note that $\lim_{x\to\infty} = \frac{2 \cosh(x)}{\exp(x)} = 1$. This implies that $\forall \delta > 0$, $\exists x > 0$ such that $2\cosh(x) < (1+\delta)\exp(x)$. If we assume that the proposition is false, then we can set $\delta' < \exp(\epsilon/m)-1$ and choose $M > 0$ such that
\begin{equation*}
\begin{split}
& (1+\delta')^m \prod_j \exp(M |\ta'_j|) > \prod_j 2 \cosh(M |\ta'_j|) \geq \prod_j 2 \cosh(M |\ta_j| ) > \prod_j \exp(M |\ta_j|) \\
\implies \,
& m \log(1+\delta') + \sum_j M | \ta'_j| > \sum_j M | \ta_j|
\, \implies \,
\epsilon > \epsilon,
\end{split}
\end{equation*}
a contradiction. Therefore, the proposition must be true.
\end{proof}


\section{Showing Modal Correspondence}

We have shown in the previous section that the modes of the joint and marginal PMF of the RBM correspond absolutely under two trivial cases: large weights and small frustration. The remaining case where the weights are small and the frustration is large is highly non-trivial, and we dedicate this entire section to showing, in the probabilistic sense, the modal correspondence as defined in definition \ref{defcor}. The problem of showing modal correspondence can be reduced to analyzing the value Gaussian integrals over simplexes of varying sizes. Before we tackle this problem, we first formalize the notion of a {\it random} RBM.

\subsection{Random RBM}

\begin{definition}[Random RBM]
\label{rand}
A {\bf Random RBM} is an RBM with a weight matrix, $\mbf{W}$, whose elements are iid normal variables with mean $\mu = 0$ and standard deviation $\sigma$. Furthermore, the configuration of the visible layer is sampled uniformly from $\{-1,+1\}^n$.
\end{definition}

\begin{lemma}
Given a random RBM, $\{ \mbf{v}, \mbf{W} \}$, the angle variables,
\begin{equation*}
\bs{\ta} = \mbf{v}\cdot \mbf{W},
\end{equation*}
are iid normal variables with mean $0$ and variance $\sit^2 = n\sigma^2$.
\end{lemma}

\begin{proof}
This is a three stage proof. First, we have to show the product $W_{ij}v_i$ is a random normal variable, so the elements of $\bs{\ta}$ are also random normal variables. Second, we show that the probability distribution function (pdf) of $\bs{\ta}$ is a multivariate normal distribution. Finally, we show that the elements of $\bs{\ta}$ are uncorrelated, thus implying that they are independent. \\

To show that $W_{ij}v_i$ is a random normal variable, we find the cumulative distribution function (CDF) of this product, and show that it is the CDF of a normal distribution. The CDF of the product is given by
\begin{equation*}
\begin{split}
& P(W_{ij}v_i <= z) \\
= & P(W_{ij} <= z)P(v_i = 1) + P(W_{ij} >= -z)P(v_i = -1) \\
= & \frac{1}{2}(P(W_{ij} <= z) + P(W_{ij} >= -z)) \\
= & \frac{1}{2}(2 P(W_{ij} <= z)) \\
= & P(W_{ij} <= z),
\end{split}
\end{equation*}
which is simply the CDF of $W_{ij}$. Note that we have exploited the fact that the PDF of $W_{ij}$ is even. Therefore, $\ta_j = \sum_i W_{ij}v_i$ is the sum of $n$ random normal variables, resulting in another random normal variable $\mathcal{N}(0,n \sigma^2)$. \\

To show that the pdf of $\bs{\ta}$ is a multivariate normal distribution, it is sufficient to show that any linear combination of the angle variables is a normal variable. Let the linear combination be
\begin{equation*}
\sum_j c_j \ta_j = \sum_j c_j (\sum_i W_{ij}v_i) = \sum_i v_i \big( \sum_j W_{ij}c_j \big).
\end{equation*}
If we denote $\phi_i = \sum_j W_{ij}c_j$, then the linear combination can be expresed as $\sum_i v_i \phi_i$. Note that we can show that $v_i\phi_i$ is a random normal variable by the same argument as above, then $\sum_i v_i\phi_i$ must be a random normal variable as well, as it is the sum of independent normal variables. Therefore, $\bs{\ta}$ is a multivariate normal distribution. \\

Finally, since the PDF of $\bs{\ta}$ is a multivariate normal distribution, to show that $\bs{\ta}$ are independent random normal variables, it is sufficient to show that any two elements of $\bs{\ta}$ are uncorrelated. For $j_1\neq j_2$, we have 
\begin{equation*}
\begin{split}
 \Cov(\ta_{j_1},\ta_{j_2}) 
= & \Cov(\sum_i W_{ij_1}v_i, \sum_i W_{ij_2}v_i) 
= \E(\sum_{i_1,i_2} W_{i_1j_1}W_{i_2j_2}v_{i_1}v_{i_2}) \\
= & \sum_{i_1,i_2} \E(W_{i_1j_1}W_{i_2j_2})\E(v_{i_1}v_{i_2}) 
= \sum_i \E(W_{ij_1}W_{ij_2}) 
= \sum_i \E(W_{ij_1})\E(W_{ij_2}) 
= \, 0,
\end{split}
\end{equation*}
where we have used the fact that $\E(v_{i_1}v_{i_2}) = \delta_{i_1i_2}$. The lemma is then proved.
\end{proof}

\begin{remark}
An important consequence of this lemma is that we can parameterize a random RBM with the angle variables $\bs{\ta}$, as the distributions of $\mbf{v}$ and $\mbf{W}$ are fully captured as the distribution of $\bsta$ as iid normal variables. \\
\end{remark}

As an RBM with large weights trivially satisfies the modal correspondence condition (see proposition \ref{largeta}), we can assume the weights are small for the sake of non-triviality, and make the following approximation for $\bsta$:
\begin{equation*}
P(\bsta) = 
\prod_j 2\cosh(\ta_j) 
\approx \prod_j (2 + \ta_j^2) 
\approx 2^m + 2^{m-1}(\sum_j \ta_j^2) 
\rightarrow \sum_j \theta_j^2,
\end{equation*}
where the right arrow in the last line denotes an affine transformation which preserves the ordering of the probability masses. Similarly, we approximate $\Qv$ as follows:
\begin{equation*}
Q(\bsta) = \exp(\sum_j |\ta_j|)\approx 1 + \sum_j |\ta_j| 
\rightarrow \sum_j |\ta_j|.
\end{equation*}

\subsection{Simplex Condition}

To show modal correspondence, it is convenient for us to fix $Q(\bsta)$, and analyze the conditional distribution of $\bsta$. In particularly, we wish to show that if $Q(\bsta)$ is large, then the conditional expected value of $P(\bsta)$ will also be large. First, we denote the conditional distribution of $\bsta$ under a fixed $\Qta$ as $f(\bs{\ta} \cond \Qta = \alpha)$. Recall that $\Qta = \sum_j |\ta_j|$ so the level set of $\Qta$ are composed of simplexes, one in each quadrant. Note that $\bsta$ are iid normal variables, so the PDF is spherically symmetric. Furthermore, $\bsta^2$ is also spherically symmetric. This means that all moments of $\Pta$ are invariant if we rewrite the condition as
\begin{equation}
\label{cond}
\big[ \Qta = \alpha \big] \,\land\, \big[ \bsta \geq \mbf{0} \big].
\end{equation}

\begin{lemma}
\label{1st}
The following two conditional distributions are equivalent.
\begin{equation*}
f\big( \bsta^2 \cond \big[ \Qta = \alpha \big] \,\land\, \big[ \bsta \geq \mbf{0} \big] \big) 
= 
f\big( \bsta^2 \cond \sum_j |\ta_j| = \alpha \big).
\end{equation*}
\end{lemma}

\begin{proof}
Omitted. Follows directly from the spherical symmetry of the PDF of $\bsta$.
\end{proof}

The graph of condition (\ref{cond}) is a regular simplex of length $\sqrt{2}\alpha$ and dimension $m-1$ in the first quadrant, which we can denote as $\simp$, and we can write the conditional PDF as $f(\bs{\ta} \cond \simp)$. It is convenient for us to apply an orthogonal transformation to $\bs{\ta}$, and we denote the new angles as $\bs{\phi} = T\bs{\ta}$. Note that the new angles are still independent normal random variables, since an orthogonal transformation preserves the independence of normal variables. The orthogonal transformation is chosen such that $\hat{\phi}_1$ points from the origin to the centroid of the simplex. We denote $\bphi$ as the components of $\bs{\varphi}$ other than $\phi_1$, meaning that $\bs{\phi} = ( \phi_1, \bs{\varphi} )$.

\begin{lemma}
Let $\mbf{n} = \frac{\alpha}{\sqrt{m}}\hat{\phi}_1$, then
\begin{equation*}
\forall \bsta \in \simp, \, \quad
\bsta^2 = (\mbf{n}^2 + \bs{\varphi}^2).
\end{equation*}
\end{lemma}

\begin{proof}
This can be shown by realizing that $\mbf{n}$ is the displacement of the centroid from the origin, which is perpendicular to the $m-1$ hyperplane the simplex is in. In other words
\begin{equation*}
\forall \bsta \in \simp, \quad
(\bsta - \mbf{n}) \cdot \mbf{n} = 0.
\end{equation*}
\end{proof}

\begin{remark}
This lemma allows us to express the condition distribution of $\bsta$ in terms of $\bs{\varphi}$.
\end{remark}

\begin{lemma}
Let $\bs{\varphi}$ be iid normal variables with variance $\sit^2$, then
\begin{equation*}
f(\bsta \cond \simp) = \left[ f(\bs{\varphi}) \middle/ \int_{\simp} f(\bs{\varphi}) \right] = f(\bs{\varphi} \cond \simp).
\end{equation*}
\end{lemma}

\begin{proof}
This can be shown by realizing that if $\bsta$ are iid normal variables, then $\bs{\phi}$ must also be iid normal variables. The intersection of the PDF of iid normal variables and a hyperplane is also a PDF of iid normal variables (with one less dimension).
\end{proof}

%
%

The conditional expected value of $\Pta$ can then be expressed as
\begin{equation}
\label{epv}
\E\big( \Pta \cond \Qta = \alpha \big) 
= \E\big( \sum_j |\ta_j|^2 \cond \simp \big)
= \frac{\alpha^2}{m} + \E( \bs{\varphi}^2 \cond \simp ).
\end{equation}
Similarly, the conditional variance can be expressed as
\begin{equation}
\label{varpv}
\Var(\Pta \cond \Qta = \alpha) = \Var( \bs{\varphi}^2 \cond \simp ).
\end{equation}
Note that the $k$-th moment of $\bs{\varphi}^2$ conditioned on the simplex is
\begin{equation*}
\E_{\simp}( \bs{\varphi}^{2k} ) = \left[ \int_{\simp} f(\bs{\varphi}) \bs{\varphi}^{2k} \middle/ \int_{\simp} f(\bs{\varphi}) \right].
\end{equation*}
To lessen the burden of notation, we denote the following Gaussian integral
\begin{equation}
\label{J}
\begin{split}
J(\sigma, \alpha, k) 
= & \int_{\simp} d\bs{\varphi} \, \bs{\varphi}^{2k} \exp\Big[ -\frac{\bs{\varphi}^2}{2\sigma^2} \Big] \\
= & \sqrt{m}\int_0^{\infty} d\bsta\, \delta( \sum_j \ta_j - \alpha ) \bs{\varphi}^{2k} \exp\Big[ - \frac{(\bsta - \mbf{n})^2}{2\sigma^2} \Big],
\end{split}
\end{equation}
and we note that
\begin{equation*}
J(\sigma,\alpha,k) 
= \Big[ \frac{\partial}{\partial\big( - \frac{1}{2 \sigma^2} \big)} \Big]^{2k} J(\sigma,\alpha),
\end{equation*}
where the last argument of $J$ is assumed $0$ in its absence. We can then write
\begin{equation*}
\E_{\simp}(\bs{\varphi}^{2k}) = \frac{1}{J(\sit, \alpha, 0)} 
\Big[ \frac{\partial}{ \partial \big( - \frac{1}{2 \sit^2} \big)} \Big]^{2k} J(\sit, \alpha, 0).
\end{equation*}

Before we proceed to evaluate this integral, we first recall that the size of the simplex is given as $\alpha = \sum_k |\ta_j|$, which means a ``typical" value of $\alpha$ is dependent on the variance, $\sit^2$. In fact, we note that $|\ta_j|$ is a half normal variation with mean $\sqrt{\frac{2}{\pi}}\sit$, then a typical size of the simplex would be
\begin{equation*}
\alpha = \sqrt{\frac{2}{\pi}}m\sit.
\end{equation*}
Therefore, when we make approximating assumptions on the integral $J$, we have to keep in mind the scaling behavior of $\alpha$ with respect to $m$ and $\sit$.

\subsection{Gaussian Integral}

We now evaluate the integral $J$ (as defined in Eq.~(\ref{J})), which we use to derive asymptotic approximations for $\E_{\simp}(\bs{\varphi}^2)$ and $\Var_{\simp}(\bs{\varphi}^2)$ in the limit of large $m$.

\begin{proposition}
We denote
\begin{equation}
\label{k}
k' = \sqrt{\frac{2}{\pi}}\Big( \, 1 - 2\sqrt{\frac{\log 2}{\pi - 2}} + \pi \sqrt{ \frac{\log 2}{\pi - 2} } \, \Big).
\end{equation}
In the limit of large $m$, we have the following linearization of $\E_{\simp}(\bs{\varphi}^2)$ and $\Var_{\simp}(\bs{\varphi}^2)$ around $k'$:
\begin{equation*}
\begin{split}
\E_{\simp}(\bs{\varphi}^2) \approx \big[ 0.727 + 0.376(k - k') \big]\,m\sit^2, \\
\Var_{\simp}(\bs{\varphi}^2) \approx \big[0.887 + 0.813(k - k') \big] \,m\sit^4.
\end{split}
\end{equation*}
\end{proposition}

\begin{proof}

We first note that
\begin{equation}
\label{moments}
\begin{split}
\E_{\simp} ( \bsv ) 
&= \frac{\sit^3 J'}{J} \\
\Var_{\simp}( \bs{\varphi}^2 )
&= \frac{3\sit^5 J' + \sit^6 J''}{J} - \frac{\sit^6 (J')^2}{J^2}.
\end{split}
\end{equation}
where the prime symbol denotes partial derivative of $J$ with respect to $\sit$. \\

We begin by transforming the integral $J(\alpha,\sit)$ in frequency space $p$
\begin{equation*}
\begin{split}
&J(\sigma, \alpha)=\frac{\sqrt{m}}{2\pi}\exp(-\frac{\alpha^2}{2m\sit^2}) \times \\
&\int_{-\infty}^{\infty}\,dp\,\exp(-ip\alpha)\Big\{ \int_0^{\infty}\,d\ta\, \exp(ip\ta - \frac{\ta^2}{2\sit^2} + \frac{\alpha\ta}{m\sit^2}) \Big\}^m \\
=& \, \big( 2^{-\frac{m}{2}-1}\pi^{\frac{m}{2}-1}\sqrt{m}\sit^m \big) \times \\
&\int_{-\infty}^{\infty}\,dp\, \exp(-\frac{1}{2}p^2\sit^2 m)\big(1+\erf(\frac{a+ipm\sit^2}{\sqrt{2}m\sit})\big)^m.
\end{split}
\end{equation*}
In order to approximate the error function, we denote $p'=\sqrt{\frac{1}{2}m\sit^2}p$ and $\lambda = \frac{\alpha}{m \sit}$. Note that $\lambda$ does not scale with $m$ or $\sit$, and its typical value is $\sqrt{\frac{2}{\pi}}$.  The integral can then be written as
\begin{equation}
\label{integral}
\begin{split}
J(\sit,\alpha) =& \big( 2^{\frac{m+1}{2}}\pi^{\frac{m}{2}-1}\sit^{m-1} \big) \times \\
& \int\,dp' \exp(-p'^2)\big( 1+\erf(\frac{i p'}{\sqrt{m}}+\frac{\lambda}{\sqrt{2}}) \big)^m.
\end{split}
\end{equation}
Note that for the argument of the error function, the real part is close to $\frac{\lambda}{\sqrt{2}} = \frac{1}{\sqrt{\pi}} < 1$, and the imaginary part approaches zero for large $m$. We then expand the error function as follows.
\begin{equation*}
\erf(x+iy) 
\,\approx\, 
\erf(x) + \frac{2i}{\sqrt{\pi}}\exp(-x^2)y + \frac{2x}{\sqrt{\pi}}\exp(-x^2)y^2,
\end{equation*}
which gives us
\begin{equation*}
\begin{split}
&\erf(\frac{i p'}{\sqrt{m}}+\frac{\lambda}{\sqrt{2}}) \\
\approx & \,\erf \big( \frac{\lambda}{\sqrt{2}} \big) 
+ \frac{2}{\sqrt{\pi}}\exp\big( -(\frac{\lambda}{\sqrt{2}})^2 \big) \big( \frac{ip'}{\sqrt{m}} \big) 
+ \frac{2}{\sqrt{\pi}}\frac{\lambda}{\sqrt{2}}\exp\big( -(\frac{\lambda}{\sqrt{2}})^2 \big)\big( \frac{p'^2}{m} \big),
\end{split}
\end{equation*}
where we kept terms only up to the order of $\lim_{m\to\infty}(1+\frac{1}{m^r})^m = 0$ as $m \rightarrow \infty$ for $r>1$. We can then approximate the $m$-th power of the above result by using the fact that
\begin{equation*}
\lim_{n\to\infty}
\frac{ \big( x+\frac{y}{\sqrt{n}} + \frac{z}{n} \big)^n }{ x^n\exp\big( \frac{z}{x}+\sqrt{n}\frac{y}{x}-\frac{1}{2}(\frac{y}{x})^2 \big)} = 1,
\end{equation*}
which allows us to write
\begin{equation*}
\begin{split}
&\big( 1+\erf(\frac{i p'}{\sqrt{m}}+\frac{\lambda}{\sqrt{2}}) \big)^m  \\
\approx & \big( 1+\erf(\frac{\lambda}{\sqrt{2}}) \big)^m 
\exp\Big[
\frac{2\sqrt{m}}{\sqrt{\pi}} C(\lambda) ip'
+ \sqrt{\frac{2}{\pi}} C(\lambda) \lambda p'^2
+ \frac{2}{\pi} C(\lambda)^2 p'^2
\Big],
\end{split}
\end{equation*}
where we have denoted
\begin{equation*}
C(\lambda) = \frac{e^{-\lambda^2/2}}{1 + \erf\big( \frac{\lambda}{\sqrt{2}} \big)}.
\end{equation*}
We then make the following approximation to the integral:
\begin{equation*}
\int_{-\infty}^{\infty}\,dp\,\exp(-p^2)\exp(ap + bp^2) = \sqrt{\frac{\pi}{1-b}}\exp\big( \frac{a^2}{4(1-b)} \big).
\end{equation*}
We can then evaluate the integral $J$ and perform the following linearization around $k'$ (the reason for the choice of $k'$ will be clear in the following subsection):
\begin{equation}
\label{lin}
\begin{split}
\E_{\simp}(\bsv^2) \approx \big( 0.727 + 0.376(k - k') \big) \,m\sit^2, \\
\Var_{\simp}(\bsv^2) \approx \big( 0.887 + 0.813(k - k') \big) \,m\sit^4.
\end{split}
\end{equation}

\end{proof}

\subsection{Density of States}

We first briefly discuss the choice of $k'$ as appeared in Eq.~(\ref{k}). We first recall that the size of the simplex,
\begin{equation*}
\alpha = \sum_{j=1}^m |\ta_j|,
\end{equation*}
is the sum of $m$ iid half-normal variables each with mean $\sqrt{\frac{2}{\pi}}\sit$ and variance $(1 - \frac{2}{\pi})\sit^2$. This means that in the limit of large $m$, $\alpha$ can be considered a normal variable with mean $\sqrt{\frac{2}{\pi}} m \sit$ and variance $(1 - \frac{2}{\pi}) m \sit^2$. We then see that in the limit of large $m$, $\alpha$ is sharply peaked at its mean (as the relative standard deviation scales as $\sqrt{\frac{1}{m}}$), as the result of the LLN (law of large numbers). This means that the probability that $\alpha$ deviates from its mean by some constant fraction scales as $e^{-m}$. \\

However, this exponential decay is compensated by the exponential increase in the number of visible configurations, which is simply $2^n$. In fact, for an $n\times n$ RBM, the contributions from the law of large numbers and entropy balance out, and a simplex whose size deviates from the typical value of $\alpha$ can still be likely generated by some visible layer configuration. We formalize this argument as follows, where $k = \frac{\alpha}{m\sit}$ is taken to be a random variable with mean $\sqrt{\frac{2}{\pi}}$ and variance $\frac{1}{n}(1-\frac{2}{\pi})$.

\begin{definition}
For a random $n\times m$ RBM, we define its {\bf density of states} at $k$, $D(n,m,k)$, to be
\begin{equation*}
D(n,m,k) = \lim_{\delta k \to 0} \frac{ \E_{\mbf{W}}\big( N(k, k+\delta k) \big) }{ \delta k} ,
\end{equation*}
where $N(k_1,k_2)$ denotes the expected number (taken over the probability measure of the weight matrix) of visible configurations that generates a simplex whose size is from $k m \sit$ to $(k + \delta k)m\sit$.
\end{definition}

\begin{remark}
For a $n\times m$ random RBM, if we take the graphs of all the simplexes generated by all the visible configurations. $D(n,m,k)$ is simply a measure of how ``densely packed" the simplexes are at $k$.
\end{remark}

\begin{proposition}
For an $n\times n$ random RBM, we denote the density of states to be $D(n,k) = D(n,n,k)$. If we let
\begin{equation*}
\begin{split}
k^o &= \sqrt{\frac{2}{\pi}}\Big( \, 1 - (\pi - 2) \sqrt{ \frac{\log 2}{\pi - 2} } \, \Big) \\
k' &= \sqrt{\frac{2}{\pi}}\Big( \, 1 + (\pi - 2) \sqrt{ \frac{\log 2}{\pi - 2} } \, \Big).
\end{split}
\end{equation*}
Then $\forall \delta k > 0$,
\begin{equation*}
\lim_{n \to \infty} D(n,k'+\delta k) = 0
\qquad
\lim_{n \to \infty} D(n,k^o-\delta k) = 0.
\end{equation*}
\end{proposition}

\begin{proof}

We first note that
\begin{equation*}
\begin{split}
D(n,k)
= & 2^n \frac{1}{\sqrt{ \frac{2\pi}{n}\big( 1 - \frac{2}{\pi} \big) }}\exp\Big[ - \frac{ \big( k - \sqrt{\frac{2}{\pi}} \big)^2 }{ \frac{2}{n}\big( 1-\frac{2}{\pi} \big) } \Big] \\
= & \sqrt{ \frac{n}{2\pi\big( 1-\frac{2}{\pi} \big)} } \Big[ \exp\big( \log(2) - \frac{ \big( k - \sqrt{\frac{2}{\pi}} \big)^2 }{ 2 \big( 1-\frac{2}{\pi} \big) } \big) \Big]^n.
\end{split}
\end{equation*}
Note that exponent evaluates to $0$ at $k = k^o$ or $k = k'$, and the exponent is negative if $k>k'$ or $k<k^o$, so the proposition follows.
\end{proof}

\begin{remark}
This proposition implies that for a random $n\times n$ RBM, the largest size of the simplex generated by the visible configuration is typically $k'$, which corresponds to the mode of the joint distribution $\mbf{v}^{\star}$. It is convenient to denote the deviation from the size of the largest simplex as $\kappa = k' - k$, and the size difference between the smallest and largest simplexes as $\delta \kappa = k' - k^o$, then under this parameterization, we can write the density of states as
\begin{equation}
D(n,\kappa)
= \sqrt{ \frac{n}{2\pi - 4} } \Big[ \exp\big( (\kappa)( \Delta\kappa - \kappa) \big) \Big]^\frac{n}{2(1-\frac{2}{\pi})}.
\end{equation}
Then from Eqs.~(\ref{epv}), (\ref{varpv}), and (\ref{lin}), we see that the conditional expected value and variance of $\Pta$ can be linearized at $\kappa = 0$ as
\begin{equation*}
\begin{split}
\E( \Pta \cond \simp ) 
&= \frac{(km\sigma)^2}{m} + \E_{\simp}(\bsv^2 \cond k) \\
&\approx \big( (0.727+k'^2) -(0.376+2k')\kappa \big)n\sit^2 \\
\sqrt{ \Var( \Pta \cond \simp ) }
&= \sqrt{ \Var\big( \bsv^2 \cond k )} \\
&\approx \big( 0.942 - 0.432 \kappa \big)\sqrt{n}\sit^2.
\end{split}
\end{equation*}
\end{remark}

If we denote $A = (0.376+2k')$ and $B(\kappa) = 0.942^2 + (0.942 - 0.432 \kappa)^2$, then for sufficiently large $\kappa > 0$, we have the following approximation
\begin{equation*}
\PP\Big[ P(\kappa) > P(0) \Big] 
= \frac{1}{2} \erfc\Big( C(n,\kappa) \Big)
\approx \frac{1}{2} \frac{1}{C(n,\kappa)\sqrt{\pi}} \exp\big[ -C(n,\kappa)^2 \big],
\end{equation*}
where we have denoted
\begin{equation*}
C(n,\kappa) = \frac{ \sqrt{n}A\kappa }{\sqrt{2 B(\kappa)}},
\end{equation*}
noting that it scales with $\sqrt{n}$. Then clearly, $\forall \delta \kappa > 0$, we have
\begin{equation*}
\lim_{m\to\infty} \frac{\erfc\big( C(n,\kappa) \big)}{ C(n,\kappa)\sqrt{\pi} }\exp\big[ -C(n,\kappa)^2 \big],
\end{equation*}
meaning that the asymptotic approximation to the $\erfc$ function is valid in the limit of large $n$. \\

If we denote an instance with the random variable $\Pta$ conditioned on the simplex $ \sum_j |\ta_j| = (k' - \kappa)n\sit$ as $P(\kappa)$, then we can make the following approximation to Eq.~(\ref{equiprob}).
\begin{equation}
\label{logint}
\begin{split}
&\log\Big[ \prod_{\mbf{v}} \PP\big( P(\mbf{v}) \leq P(\mbf{v}^{\star} ) \big) \Big] \\
=& \sum_{\mbf{v}} \log\Big[ \PP\big[ P(\mbf{v}) \leq P(\mbf{v}^{\star})\big] \Big] \\
\approx & -\int_{\delta\kappa}^{\infty}\,d\kappa\,D(n,\kappa) \frac{1}{2} \frac{1}{C(n,\kappa)\sqrt{\pi}} \exp\big[ -C(n,\kappa)^2 \big],
\end{split}
\end{equation}
where $\delta\kappa > 0$ is some small constant. We formalize this approximation as follows.

\begin{proposition}
Let $\delta\kappa = \frac{1}{n}$, then in the limit of large $n$,
\begin{equation*}
\int_{0}^{+\infty}\, d\kappa \, D(n,\kappa) \log\Big[ \PP\big( P(\kappa) \leq P(0) \big) \Big] 
\approx -\int_{\delta\kappa}^{+\infty}\,d\kappa\,D(n,\kappa) \frac{1}{2} \frac{1}{C(n,\kappa)\sqrt{\pi}} \exp\big[ -C(n,\kappa)^2 \big].
\end{equation*}
\end{proposition}

\begin{proof}

We denote the following integral
\begin{equation}
I(n,\kappa_1,\kappa_2) = \int_{\kappa_1}^{\kappa_2}\, d\kappa \, D(n,\kappa) \log\Big[ \PP\big( P(\kappa) \leq P(0) \big) \Big].
\end{equation}
First thing to note is that the integrand is always negative as the density of states, $D(n,\kappa)$, is necessarily positive, and the log-likelihood is necessarily negative. We then break $I(n,-\infty,+\infty)$ into three parts:
\begin{equation*}
I(n,0,+\infty) = 
I(n,0,+\delta\kappa) 
+ I(n,+\delta\kappa,+\infty).
\end{equation*}
To prove the proposition, it is sufficient to show in the limit of large $n$ that the integral goes to zero for the first and last terms, and the asymptotic approximation for the error function is valid for the second term. \\

For the integral $I(n,+\delta\kappa,+\infty)$, we have $\kappa \geq +\delta\kappa$, and we can approximate the log-likelihood as
\begin{equation*}
\begin{split}
& \log \Big( \PP\big( P(\kappa) \leq P(0) \big) \Big) 
= \log \Big( 1 - \PP\big( P(\kappa) > P(0) \big) \Big) 
\approx - \frac{1}{2} \erfc\Big( \frac{\sqrt{n}A\kappa}{\sqrt{2B(\kappa)}} \Big) \\
\geq & - \frac{1}{2} \erfc\Big( \frac{ 3A }{\sqrt{2nB(0)}} \Big) 
\rightarrow 0,
\end{split}
\end{equation*}
as $n$ goes to infinity, meaning that the $\erfc$ approximation is valid. \\

For the integral $I(m,0,+\delta\kappa)$, we have $\kappa \geq +\delta\kappa$, and we obtain the following
\begin{equation*}
\log\Big( \PP\big( P(\kappa) \leq P(0) \big) \Big) \geq \log\big( \frac{1}{2} \big) \approx -0.69.
\end{equation*}
Recall that
\begin{equation*}
D(n,\kappa)
= \sqrt{ \frac{n}{2\pi - 4} } \Big[ \exp\big( (\kappa)( \Delta\kappa - \kappa) \big) \Big]^\frac{n}{2(1-\frac{2}{\pi})},
\end{equation*}
which means
\begin{equation*}
|I(n,0,+\delta\kappa)| \geq 0.69 \int_{0}^{+\delta\kappa} d\kappa\, D(n,\kappa) \rightarrow 0,
\end{equation*}
as we take $n$ to infinity. The proposition is then shown.

\end{proof}

\begin{corollary}
For sufficiently small values of $n$, the joint and marginal modes of a random $n\times n$ RBM are strongly correlated, under definition \ref{defcor}.
\end{corollary}

\begin{proof}
This can be shown by directly evaluating the logarithm of the integral as given in Eq.~(\ref{logint}), and verify that the result is greater than $\log(\frac{1}{2})$, up to a certain value of $n_{\max}$.
\end{proof}


\end{document}